\documentclass[11pt]{article}

\usepackage[T1]{fontenc}
\usepackage{lmodern}
\usepackage[margin=1in]{geometry}
\usepackage{microtype}
\usepackage{amsmath,amssymb,amsthm,mathtools}
\usepackage{booktabs,float}
\usepackage{algorithm,algorithmic}
\usepackage{dsfont}
\usepackage{xcolor}

\usepackage{authblk}
\usepackage{amsmath,amsfonts,bm}

\def\eqref#1{equation~\ref{#1}}
\def\1{\bm{1}}

\def\vy{{\bm{y}}}

\DeclareMathAlphabet{\mathsfit}{\encodingdefault}{\sfdefault}{m}{sl}
\SetMathAlphabet{\mathsfit}{bold}{\encodingdefault}{\sfdefault}{bx}{n}

\def\gA{{\mathcal{A}}}

\def\gD{{\mathcal{D}}}
\def\gE{{\mathcal{E}}}
\def\gF{{\mathcal{F}}}

\def\gO{{\mathcal{O}}}

\def\gR{{\mathcal{R}}}
\def\gS{{\mathcal{S}}}

\def\gX{{\mathcal{X}}}
\def\gY{{\mathcal{Y}}}

\newcommand{\E}{\mathbb{E}}

\newcommand{\KL}{D_{\mathrm{KL}}}

\usepackage[authoryear,round]{natbib}
\usepackage{url}

\newtheorem{theorem}{Theorem}[section]
\newtheorem{lemma}{Lemma}[section]

\theoremstyle{definition}
\newtheorem{definition}{Definition}[section]

\newcommand{\ind}{\mathds{1}}
\newcommand{\otob}{\widehat{\pi}_{\mathrm{o2b}}}
\DeclareMathOperator{\TV}{TV}
\DeclareMathOperator{\kl}{kl}

\newcommand{\demo}{\pi_{\mathrm{demo}}}
\newcommand{\thdemo}{\pi_{\mathrm{demo}}^{(\theta)}}
\newcommand{\sgdemo}{\pi_{\mathrm{demo}}^{(\sigma)}}

\definecolor{cm}{rgb}{0,0,0.78}

\usepackage{hyperref}
\hypersetup{
    colorlinks=true,
    allcolors=blue,
    pdftitle={The Statistical Benefits of Multiple Responses for Learning from Demonstrations},
    pdfauthor={Chandramauli Chakraborty and Cong Ma}
}
\usepackage[nameinlink,capitalise]{cleveref}

\title{The Statistical Benefits of Multiple Responses\\
       for Learning from Demonstrations}
\author{Chandramauli Chakraborty}
\author{Cong Ma}
\affil{Department of Statistics, University of Chicago\\
\texttt{\{chandramaulic, congm\}@uchicago.edu}}

\date{}

\begin{document}

\maketitle
\begin{abstract}
Many generative systems return multiple candidate responses and are evaluated according to the best one. Recent work shows that, when demonstrations are optimal, pass@$k$ can reduce the sample complexity of learning from demonstrations by a logarithmic factor in $k$. We ask what happens when the demonstrator is not assumed to be optimal.
We find that multiple responses provide a qualitatively stronger benefit in this setting. In a finite reward-class model with no reward feedback, moving from pass@$1$ to any pass@$k$ with $k\ge2$ changes the worst-case dependence on target accuracy from $1/\varepsilon^2$ to $1/\varepsilon$, uniformly over demonstrator quality. Under standard evaluation, where an unknown reward is fixed before training, increasing $k$ provides an additional and distinct benefit: the optimal dependence on a reward class of size $N$ improves from $\log N$ to $\log N/\log k$. We further show that these two effects can be separated. Under robust evaluation, where one learned policy must compete with the demonstrator simultaneously for every reward in the class, the fast $1/\varepsilon$ dependence persists, while the $1/\log k$ improvement can disappear. We establish matching upper and lower bounds in the corresponding regimes and give a greedy multiplicative-weights learner achieving the upper bounds without any assumption on demonstrator quality.
\end{abstract}
% \noindent\textbf{Keywords:}
% Learning from demonstrations; pass@$k$; suboptimal demonstrators;
% sample complexity; minimax rates; robust evaluation.
% \newpage
% \tableofcontents
\section{Introduction}
\label{sec:intro}

% Paragraph 1 theme: Motivation. Learning from demonstrations is usually studied
% under single-response evaluation, but modern generative systems often return
% multiple candidates. State the basic question of the paper without introducing
% technical notation yet.

Learning from demonstrations~\citep{argall2009survey,ouyang2022training,joshi2026learning} is a common paradigm for training modern predictive and generative systems:
for each input or context, the learner observes a response produced by a human, an expert, or another model, and seeks to perform comparably well without direct access to the underlying reward that makes those responses desirable. This problem falls within apprenticeship learning~\citep{abbeel2004apprenticeship,syed2007game}, a branch of imitation learning~\citep{pomerleau1988alvinn, ng2000algorithms} that studies how to match a demonstrator's value under an unknown reward. Most of the learning-theoretic literature has focused on the usual pass@1 setting, where performance is determined by a single returned response. Yet modern generative systems are often evaluated using pass@\(k\): a model produces several candidate answers, programs, or solutions, and performance is determined by the best one~\citep{chen2021evaluating,li2022competition,brown2024large}. Recent work also directly optimizes this pass@$k$ metric during reinforcement
learning \citep{tang2025optimizing,walder2026pass,chen2025pass,bagirov2025best}. Despite its practical relevance, the statistical effect of moving from pass@1 to pass@\(k\) remains much less understood.

% Paragraph 2 theme: Prior work and the gap. Explain what Joshi et al. established,
% why optimal demonstrations are a favorable special case, and identify arbitrary
% demonstrations as the missing piece. End with the two questions the paper answers.

Recent work by~\cite{joshi2026learning} initiated the
learning-theoretic study of pass@$k$ in this apprenticeship-learning
framework. They considered the favorable setting of optimal demonstrations
and showed that, for a reward class of cardinality $N$, pass@$k$ improves
the dependence of the sample complexity on $N$ from order $\log N$ to
order $\log N/\log k$, with a matching lower bound. In practice, however,
demonstrations need not be optimal: human demonstrators may have limited
expertise or make mistakes, and demonstrations generated by other models
can themselves be imperfect. More generally, high-quality demonstrations
need not maximize the reward ultimately used for evaluation. Without assuming that the demonstrator is optimal, we ask how allowing the learner to return multiple responses rather than a single response changes the sample complexity of learning from demonstrations.

% Paragraph 3 theme: Informal setup and the two evaluation modes. Give only enough
% notation to make the results understandable. Do not define PAC sample complexity
% here; leave the formal definitions to Section 2.

To address this question, we consider a finite reward-class model in which the learner observes $m$ context--response pairs generated by an unknown demonstrator $\demo$ and knows a finite class $\mathcal R$ of bounded reward functions, but receives no reward feedback. At test time, a pass@$k$ policy returns $k$ possibly dependent responses and is evaluated according to the highest reward among them. We compare its value with the demonstrator's single-response value.

We study two natural evaluation criteria. The first is the \emph{standard evaluation} setting of~\citet{joshi2026learning}, in which an unknown reward \(r^\star\in\mathcal R\) is fixed before training and performance is measured under that reward. 
Secondly, when several reward functions are plausible descriptions of the desired behavior, it may be undesirable to tie the guarantee to one particular reward. This motivates the stronger \emph{robust evaluation} criterion, under which the learned policy must compete with the demonstrator simultaneously for every reward in $\mathcal R$~\citep{syed2007game}. Throughout the introduction, we focus on arbitrary demonstrators, without any optimality or near-optimality assumption.\footnote{In the formal development, this corresponds to the unrestricted suboptimality budget $\Delta=1$. Section~\ref{sec:results} gives results for general $\Delta\in[0,1]$.}

\paragraph{Standard evaluation.}
Under standard evaluation, we obtain a sharp characterization for every $k\ge2$. Write \(N:=|\mathcal R|\), and let \(\varepsilon\) and \(\delta\) denote the target accuracy and failure probability, respectively. The worst-case sample complexity scales as
\[
\frac{1}{\varepsilon}
\left(
\frac{\log N}{\log k}
+
\log\frac1\delta
\right)
.
\]
The upper bound holds for arbitrary demonstrators, while the matching lower bound already holds when the demonstrator is optimal for the selected evaluation reward.

Comparing this result with pass@$1$ reveals two statistically distinct benefits of multiple responses. For arbitrary demonstrators, the pass@$1$ upper bound of \citet{joshi2026learning} has order
\[
\frac{1}{\varepsilon^2} \left(\log N + \log \frac{1}{\delta}\right),
\]
and we show that this dependence is unimprovable in general. Hence the first benefit is a qualitative slow-to-fast transition: moving from one response to any \(k\ge2\) changes the worst-case dependence on accuracy from $1/\varepsilon^2$ to $1/\varepsilon$. The mechanism is simple: when two responses are separated only by a small bias in the demonstrations, pass@$1$ must estimate which response is better, whereas pass@$2$ can cover both possibilities. Multiple responses therefore turn an estimation problem into a coverage problem. The second benefit concerns the dependence on the reward-class size. \citet{joshi2026learning} showed, for optimal demonstrations, that increasing \(k\) improves this dependence from \(\log N\) to \(\log N/\log k\). Our result shows that this improvement persists even for arbitrarily suboptimal demonstrations.

% Paragraph 6 theme: Robust evaluation. Show that it separates the two benefits:
% the slow-to-fast transition survives, but the 1/log k gain does not in the
% high-accuracy regime. State only sample complexities.
\paragraph{Robust evaluation.}
The picture changes under robust evaluation. 
For every \(k\ge2\), the same learner achieves, for arbitrary demonstrators, sample complexity
\[
\frac{1}{\varepsilon} \left(\log N + \log \frac{1}{\delta}\right),
\]
up to universal constants, and we prove a matching lower bound in the high-accuracy regime, i.e., when $\varepsilon \ll 1/k$. 
Thus the slow-to-fast transition from pass@$1$ to pass@$k$ survives the stronger robust guarantee. In contrast, the additional \(1/\log k\) improvement in the reward-class dependence can disappear: in this regime, a full \(\log N\) dependence is unavoidable regardless of \(k\). 
Table~\ref{tab:intro-rates} summarizes the comparison of optimal sample complexities for arbitrary
demonstrators.

\begin{table}[tbh]
\centering
\caption{
Sample complexity for arbitrary demonstrators, up to universal constants. See Section~\ref{sec:results} and Appendix~\ref{appen:tables} for precise statements.
}
\label{tab:intro-rates}
\small
\renewcommand{\arraystretch}{1.45}
\begin{tabular}{@{}lcc@{}}
\toprule
&
\textbf{Standard evaluation}
&
\textbf{Robust evaluation}
\\
\midrule
pass@$1$
&
$\displaystyle
\frac{\log N+\log(1/\delta)}{\varepsilon^2}
$
&
$\displaystyle
\frac{\log N+\log(1/\delta)}{\varepsilon^2}
$
\\[3mm]
pass@$k$, $k\ge2$
&
$\displaystyle
\frac{1}{\varepsilon}
\left(
\frac{\log N}{\log k}
+\log\frac1\delta
\right)
$
&
$\displaystyle
\frac{\log N+\log(1/\delta)}{\varepsilon}
$
\\
\bottomrule
\end{tabular}
\end{table}

\subsection{Related work}

% \cm{Right now, this is repeating what we said in the intro.}
% \subsection{Related work}
% \label{sec:related_work}
\paragraph{Learning from demonstrations.}
Apprenticeship learning aims to match a demonstrator's value under an
unknown reward \citep{abbeel2004apprenticeship,syed2007game}.
Behavior-cloning analyses often assume a restricted policy class
\citep{foster2024behavior}; our learner instead uses a known reward
class and does not fit the demonstrator's response distribution.
Closest to our setting, the pass@$k$ guarantees of
\citet{joshi2026learning} assume optimal demonstrations.
We remove this assumption while retaining fast $1/\varepsilon$
sample complexity for every $k\ge2$, and also give guarantees that
hold simultaneously over the reward class.
\citet{pour2026learning} characterize realizable online learning with
multiple correct answers under different feedback models. Their
realizable setting with one correct response and no feedback on the
learner's correctness is close to our setting with binary rewards and
optimal demonstrations. They study single-response prediction and
use a different agnostic benchmark from our comparison with the
demonstrator's value.

\paragraph{List learning.}
PAC and online list learning allow several predicted labels and count
an error when the target label is absent
\citep{charikar2023characterization,moran2023list}.
PAC learning uses i.i.d.\ examples, while online learning receives
label feedback after each prediction. Our pass@$k$ objective also
evaluates several responses, but compares their best reward with the
demonstrator's value. Several responses may be acceptable, rewards need
not be binary, and demonstrations may be suboptimal. Thus a successful
prediction need not contain the demonstrated response.

\paragraph{Inference-time selection.}
Recent work studies how to select candidates from a given generator
using an estimated reward model
\citep{huang2025best,di2026best}.
Their guarantees depend on the generator's ability to produce good
responses, the accuracy of reward estimates, and the number of
candidates sampled. We instead learn the policy that generates the
candidates, using demonstrations and a known reward class. The learner
receives no reward feedback during training. Our guarantees concern
the number of demonstrations needed to match the demonstrator's value
under pass@$k$ evaluation.t

\section{Problem Setup and Main Results}\label{sec:results}

\subsection{Setup and evaluation criteria}

Let $\gX$ and $\gY$ denote the context and response spaces, and let
$\gD\in\Delta(\gX)$ be an unknown context distribution, where
$\Delta(\mathcal Z)$ denotes the set of probability distributions on
$\mathcal Z$.  The learner knows a finite, non-empty reward class
$\gR\subseteq[0,1]^{\gX\times\gY}$, but observes neither the evaluation
reward nor reward feedback.  An unknown demonstrator
$\demo:\gX\to\Delta(\gY)$ generates $m$ i.i.d.\ demonstrations
\[
    S=((x_1,y_1),\ldots,(x_m,y_m)),
    \qquad
    x_i\sim\gD,\quad
    y_i\sim\demo(\cdot\mid x_i).
\]
A possibly randomized learning rule $\mathcal{A}_k(\gR, S)$ returns a pass@$k$ policy
$\widehat\pi=\mathcal{A}_k(\gR, S):\gX\to\Delta(\gY^k)$, whose $k$ responses may be dependent.
For $r\in\gR$, define
\[
    V_r^k(\pi)
    :=
    \E_{x\sim\gD}
    \E_{(y^{(1)},\ldots,y^{(k)})\sim\pi(\cdot\mid x)}
    \left[
        \max_{\ell\in[k]}r(x,y^{(\ell)})
    \right],
\]
and write $V_r:=V_r^1$.  We measure performance relative to the
demonstrator through
\[
    \gE_r^k(\pi,\demo)
    :=
    V_r(\demo)-V_r^k(\pi).
\]
The excess value may be negative when the learned policy outperforms the
demonstrator.  

We study two evaluation modes: standard evaluation at a fixed but unknown
reward, and robust evaluation simultaneously over the reward class.

\paragraph{Standard evaluation.}
Under standard evaluation, a target reward $r^\star\in\gR$ is fixed
before training but is not revealed to the learner.  The goal is to make
$\gE_{r^\star}^k(\widehat\pi,\demo)$ small.

The demonstrator need not be optimal for $r^\star$.  Define the optimal
value under reward $r$ by
\[
    V_r^\star
    :=
    \E_{x\sim\gD}
    \left[\sup_{y\in\gY}r(x,y)\right],
\]
and define the demonstrator's suboptimality under $r$ as
\[
    \Delta_r
    :=
    V_r^\star-V_r(\demo).
\]
We use $\Delta\in[0,1]$ as a budget on this quantity.  Thus
$\Delta=0$ restricts attention to demonstrators that are optimal for
the target reward, whereas $\Delta=1$ places no restriction on the
demonstrator.

\begin{definition}[Standard sample complexity]
\label{def:standard_pac_complexity}
For $\varepsilon,\delta\in(0,1)$ and $\Delta\in[0,1]$, the standard
sample complexity $m_{\mathrm{std}}^{(k)}
    (\gR,\varepsilon,\delta,\Delta)$ 
is the smallest nonnegative integer $m_0$ for which there exists a
learning rule such that, for every $m\ge m_0$, every context
distribution $\gD$, every target reward $r^\star\in\gR$, and every
demonstrator $\demo$ satisfying
$
    \Delta_{r^\star}\le\Delta,
$
the learned policy satisfies
\[
    \mathbb P\!\left\{
        \gE_{r^\star}^k(\widehat\pi,\demo)>\varepsilon
    \right\}
    \le\delta.
\]
\end{definition}

\paragraph{Robust evaluation.}
Under robust evaluation, a single learned policy must compete with the
demonstrator simultaneously for every reward in $\gR$.  Accordingly,
the relevant notion of demonstrator quality is its worst-case
suboptimality over the reward class,
\[
    \Delta_{\gR}
    :=
    \sup_{r\in\gR}\Delta_r
    =
    \sup_{r\in\gR}
    \bigl[V_r^\star-V_r(\demo)\bigr].
\]

\begin{definition}[Robust sample complexity]
\label{def:robust_pac_complexity}
For $\varepsilon,\delta\in(0,1)$ and $\Delta\in[0,1]$, the robust
sample complexity $m_{\mathrm{rob}}^{(k)}
    (\gR,\varepsilon,\delta,\Delta)$ 
is the smallest nonnegative integer $m_0$ for which there exists a
learning rule such that, for every $m\ge m_0$, every context
distribution $\gD$, and every demonstrator $\demo$ satisfying $\Delta_{\gR}\le\Delta,$
the learned policy satisfies
\[
    \mathbb P\!\left\{
        \sup_{r\in\gR}
        \gE_r^k(\widehat\pi,\demo)>\varepsilon
    \right\}
    \le\delta.
\]
\end{definition}

Both sample complexities above are defined for a fixed reward class \(\mathcal R\). Our upper bounds hold for every finite \(\mathcal R\), whereas the lower bounds below show that, for each prescribed cardinality \(N\), there exist reward classes of size \(N\) (or at most \(N\)) with the stated sample complexity. We first study standard evaluation and then turn to robust evaluation to determine which benefits of multiple responses survive the stronger requirement of simultaneous validity.

\subsection{Standard evaluation}

Under standard evaluation, the learner needs to compete with the demonstrator under a single reward fixed before training, although that reward is never revealed. We first show that, once \(k\ge2\), the sample complexity no longer depends on the quality of the demonstrator.

\paragraph{Pass@$k$ achieves a fast rate for arbitrary demonstrators.}
Our first result gives a high-probability guarantee that holds without
any assumption on demonstrator quality.
The learning rule achieving this guarantee is constructed and analyzed in Section~\ref{sec:optimal_learning}.
\begin{theorem}[Standard pass@$k$ upper bound]
\label{thm:standard_passk_upper}
For every $k\ge2$, there exists a learning rule $\gA_k$ with
the following guarantee. For every $m\ge1$, $\delta\in(0,1)$,
context distribution $\gD$, demonstrator $\demo$, and
evaluation reward $r^\star\in\gR$ fixed before sampling,
the policy $\widehat\pi=\gA_k(\gR,S)$ satisfies
\[
    \mathbb P_S\!\left\{
        \gE_{r^\star}^k(\widehat\pi,\demo)
        >
        \frac{3\log|\gR|}{m\log k}
        +\frac{3\log(1/\delta)+1}{m}
    \right\}
    \le\delta.
\]

Consequently, for every $\varepsilon>0$, $\delta\in(0,1)$,
and $\Delta\in[0,1]$,
\[
    m_{\mathrm{std}}^{(k)}
        (\gR,\varepsilon,\delta,\Delta)
    \le
    \left\lceil
        \frac{3\log|\gR|/\log k+3\log(1/\delta)+1}
             {\varepsilon}
    \right\rceil.
\]
\end{theorem}

Two features of the bound are worth separating. First, its \(1/\varepsilon\) dependence is independent of \(\Delta\): no additional samples are required when the demonstrator is imperfect. Second, increasing \(k\) improves the dependence on the reward-class size from \(\log|\mathcal R|\) to \(\log|\mathcal R|/\log k\), while leaving the confidence term unchanged.

The upper bound leaves two questions: is the independence from demonstrator quality genuine, and is the \(1/\log k\) dependence optimal? The next result answers both affirmatively. In fact, the matching lower bound already holds for demonstrators that are optimal for the evaluation reward.

\begin{theorem}[Standard pass@$k$ lower bound]
\label{thm:standard_passk_lower}
There exist universal constants $c,c_0>0$ such that the
following holds. For every pair of integers $k\ge2$ and
$N\ge2k$, there exist finite  spaces $\gX,\gY$ and
a binary reward class
$\gR\subseteq\{0,1\}^{\gX\times\gY}$ with $|\gR|=N$ such that,
for every $\Delta\in[0,1]$, $0<\varepsilon\le c_0$, and
$0<\delta<1/16$,
\[
    m_{\mathrm{std}}^{(k)}
        (\gR,\varepsilon,\delta,\Delta)
    \ge
    \frac{c}{\varepsilon}
    \left(
        \frac{\log N}{\log k}
        +\log\frac1\delta
    \right).
\]
The lower bound holds even when the demonstrator is optimal
for the selected evaluation reward.
\end{theorem}

Together, Theorems~\ref{thm:standard_passk_upper}
and~\ref{thm:standard_passk_lower} establish the optimal
dependence on reward-class size, accuracy, and confidence
in the stated parameter range. The matching bounds show that neither the fast rate nor the \(1/\log k\) improvement relies on favorable demonstrations. The upper bound holds for arbitrary demonstrators, while the lower bound already holds for optimal ones. Thus, once multiple responses are allowed, demonstrator suboptimality creates no additional statistical cost. The proofs of Theorem~\ref{thm:standard_passk_upper} and Theorem~\ref{thm:standard_passk_lower} are given in the Appendices~\ref{appen:standard_passk_upper} and~\ref{appen:standard_passk_lower} respectively.

\paragraph{Comparison with pass@$1$.}
To isolate what is gained by allowing multiple responses, we now compare the preceding rate with the single-response problem.  A result of~\citet{joshi2026learning} gives
\[
    m_{\mathrm{std}}^{(1)}(\gR,\varepsilon,\delta,\Delta)
    \lesssim
    \frac{\log|\gR|+\log(1/\delta)}{\varepsilon}
    +\frac{\Delta\,[\log|\gR|+\log(1/\delta)]}{\varepsilon^2}.
\]
For an optimal demonstrator (\(\Delta=0\)), this bound has the same fast \(1/\varepsilon\) dependence as pass@\(k\). Positive suboptimality, however, introduces a potentially dominant \(1/\varepsilon^2\) term. The following lower bound shows that this deterioration is intrinsic rather than an artifact of the analysis. 

\begin{theorem}[Standard pass@$1$ lower bound]
\label{thm:standard_pass1_lower}
There exists a universal constant $c>0$ such that the
following holds. For every integer $N\ge2$,
$\Delta\in[0,1]$, $0<\varepsilon\le1/64$, and
$0<\delta\le1/64$, there exist finite  spaces
$\gX,\gY$ and a finite  binary reward class
$\gR\subseteq\{0,1\}^{\gX\times\gY}$ with $|\gR|= N$
such that
\[
\begin{aligned}
    m_{\mathrm{std}}^{(1)}
        (\gR,\varepsilon,\delta,\Delta)
    \ge c\Bigg[
        &\frac{\log N+\log(1/\delta)}{\varepsilon}+\frac{\Delta}{\varepsilon^2}
        \left\{
            \log\frac1\delta
            +\log\left(
                1+\frac{N\varepsilon^2}
                         {(\Delta+\varepsilon)^2}
            \right)
        \right\}
    \Bigg].
\end{aligned}
\]
\end{theorem}

At $\Delta=0$, the lower bound has the fast order
$[\log N+\log(1/\delta)]/\varepsilon$.
For positive suboptimality budgets, it also captures a
slow contribution proportional to $\Delta/\varepsilon^2$.
In particular, if
\[
    N\ge(1+\Delta/\varepsilon)^4
    \qquad\text{or}\qquad
    \log(1/\delta)\ge\log N,
\]
the lower bound matches the suboptimality-dependent upper bound up to
universal constants. See the proof in Appendix~\ref{appen:standard_pass1_lower}.

The preceding bounds make precise the two benefits of multiple responses
highlighted in the introduction. First, there is a qualitative
slow-to-fast transition at $k=2$: for arbitrary demonstrators, pass@$1$
can require $1/\varepsilon^2$ samples, whereas every pass@$k$ with
$k\ge2$ achieves a $1/\varepsilon$ rate. Second, within the regime
$k\ge2$, increasing $k$ further improves the dependence on the
reward-class size from $\log N$ to $\log N/\log k$, extending the
optimal-demonstrator phenomenon of \citet{joshi2026learning} to
arbitrarily suboptimal demonstrations.

The first transition reflects a change in the underlying statistical
task. This can already be seen with a single context and two responses
$a$ and $b$. Consider two possible instances: in the first, only $a$
receives reward one and the demonstrator outputs $a$ with probability
$1/2+2\varepsilon$; in the second, only $b$ receives reward one and the
demonstrator outputs $b$ with the same probability. A pass@$1$ learner
must distinguish these two possibilities. Predicting $a$ and $b$ equally
often achieves value $1/2$, while the demonstrator's value is
$1/2+2\varepsilon$, so attaining excess at most $\varepsilon$ requires
identifying the sign of an $\gO(\varepsilon)$ bias. This is a standard
small-bias testing problem and requires order $1/\varepsilon^2$ samples.

A pass@$2$ learner can instead return both responses and obtain reward one
under either instance, without estimating the bias at all. What remains is
a coverage problem: the learner must ensure that relevant response
possibilities are represented among its outputs. This is closely related
to missing-mass phenomena~\citep{rashidinejad2021bridging,rajaraman2020toward}, where the error is controlled by the probability
mass of possibilities not yet covered by the sample, leading to a
$1/m$ rather than a $1/\sqrt m$ accuracy scale. Thus multiple responses
remove the small-bias estimation bottleneck and replace it by a
coverage-type difficulty.

\subsection{Robust evaluation}
\label{subsec:robust_evaluation}
Standard evaluation revealed two distinct benefits of multiple responses: moving from \(k=1\) to \(k\ge2\) removes the suboptimality-induced slow rate, while increasing \(k\) further improves the dependence on the reward-class size by a factor of \(\log k\). Robust evaluation allows us to ask whether these two benefits persist when a single learned policy must perform well simultaneously for every reward in \(\mathcal R\).

\paragraph{The $1/\log k$ improvement disappears.}

The same learning rule used under standard evaluation also gives a robust
guarantee.

\begin{theorem}[Robust pass@$k$ upper bound]
\label{thm:robust_passk_upper}
Let $k\ge2$. The learning rule $\gA_k$ from
Theorem~\ref{thm:standard_passk_upper} satisfies the
following guarantee: for every $m\ge1$, $\delta\in(0,1)$,
context distribution $\gD$, and demonstrator $\demo$,
the policy $\widehat\pi=\gA_k(\gR,S)$ obeys
\[
    \mathbb P_S\!\left\{
        \sup_{r\in\gR}
        \gE_r^k(\widehat\pi,\demo)
        >
        \frac{3\log|\gR|}{m\log k}
        +\frac{3\log(|\gR|/\delta)+1}{m}
    \right\}
    \le\delta.
\]
Consequently, for every $\varepsilon>0$,
$\delta\in(0,1)$, and $\Delta\in[0,1]$,
\begin{equation}
\label{eq:robust_passk_upper_samples}
    m_{\mathrm{rob}}^{(k)}
        (\gR,\varepsilon,\delta,\Delta)
    \le
    \left\lceil
        \frac{8\log(|\gR|/\delta)+1}{\varepsilon}
    \right\rceil.
\end{equation}
\end{theorem}

The same learner therefore retains the fast \(1/\varepsilon\) dependence under robust evaluation. The difference is that simultaneous validity introduces an additional \(\log|\mathcal R|\) term that is not reduced as \(k\) grows. From the upper bound alone, however, it is unclear whether this loss is fundamental or merely the price of the union-bound analysis. The next result shows that it is fundamental in the high-accuracy regime.

\begin{theorem}[Robust pass@$k$ lower bound]
\label{thm:robust_passk_lower}
For every pair of integers $k\ge2$ and $N\ge2k^2$,
there exist finite spaces $\gX,\gY$ and a
binary reward class
$\gR\subseteq\{0,1\}^{\gX\times\gY}$ with $|\gR|=N$
such that, for every $0<\Delta\le1$,
$0<\delta<1/16$, and $0<\varepsilon
    \le \frac{1}{384}\min\{\Delta,1/k\},$
we have
\begin{equation}
\label{eq:robust_passk_lower_samples}
    m_{\mathrm{rob}}^{(k)}
        (\gR,\varepsilon,\delta,\Delta)
    >
    \left\lfloor
        \frac{\log N+\log(1/\delta)}
             {384\varepsilon}
    \right\rfloor.
\end{equation}
\end{theorem}

Combining the two bounds, for $N\ge2k^2$ and $\varepsilon
    \lesssim
    \min\!\left\{\Delta,\frac1k\right\},$
we obtain the sharp sample complexity  $(\log N+\log(1/\delta))/\varepsilon$. 
Thus robust evaluation cleanly separates the two benefits identified under standard evaluation. The transition to a fast \(1/\varepsilon\) rate survives, but the additional \(1/\log k\) improvement does not: in the high-accuracy regime, a full \(\log N\) dependence is unavoidable regardless of \(k\). Since the same learner achieves the standard and robust upper bounds, this difference is caused by the stronger evaluation requirement rather than by a change in algorithm. The proofs of Theorem~\ref{thm:robust_passk_upper} and Theorem~\ref{thm:robust_passk_lower} are given in Appendices~\ref{appen:robust_passk_upper} and~\ref{appen:robust_passk_lower}, respectively.

\paragraph{Comparison with pass@$1$.}
The disappearance of the \(1/\log k\) gain raises a natural question: does allowing multiple responses still help at all under robust evaluation? The answer is yes. To see this, we compare the preceding rate with pass@\(1\).
Although the pass@$1$ result of~\citet{joshi2026learning} is stated
under standard evaluation, their analysis actually gives a simultaneous
high-probability guarantee over the reward class.  It therefore implies
the robust upper bound
\[
    m_{\mathrm{rob}}^{(1)}
    (\gR,\varepsilon,\delta,\Delta)
    \lesssim
    \frac{\log(|\gR|/\delta)}{\varepsilon}
    +
    \frac{\Delta\log(|\gR|/\delta)}{\varepsilon^2}.
\]
We show that the slow term is unavoidable.

\begin{theorem}[Robust pass@$1$ lower bound]
\label{thm:robust_pass1_lower}
For every integer $N\ge2$ and $0<\delta<1/16$,
there exist finite spaces $\gX, \gY$ and a finite
binary reward class $\gR$ with $|\gR|= N$ such that,
for every $0<\Delta\le1$ and
$0<\varepsilon\le\Delta/256$,
\begin{equation}
\label{eq:robust_pass1_lower_samples}
    m_{\mathrm{rob}}^{(1)}
        (\gR,\varepsilon,\delta,\Delta)
    >
    \left\lfloor
        \frac{\Delta\,[\log N+\log(1/\delta)]}
             {262144\,\varepsilon^2}
    \right\rfloor.
\end{equation}
\end{theorem}

When $\varepsilon\le c_0\Delta$ for some sufficiently small constant $c_0$, the upper bound
has order
$\Delta[\log|\gR|+\log(1/\delta)]/\varepsilon^2$.
The lower bound shows that this dependence is optimal
in general. Taking $\Delta=1$ gives the
arbitrary-demonstrator rate
$[\log N+\log(1/\delta)]/\varepsilon^2$. Comparing this with the pass@\(k\) rate shows that the slow-to-fast transition survives robust evaluation, even though the additional \(1/\log k\) improvement does not. See the proof in Appendix~\ref{appen:robust_pass1_minimax}.\footnote{
At $\Delta=0$, a stronger conclusion holds:
no demonstrations are needed. At each context,
choose a response maximizing the sum of the known
rewards. Under class-wide optimality, this policy is
optimal for every reward on a set of contexts of
probability one. Further explanation is given in the Appendix~\ref{appen:sec_unif_opt_demon}.}. All sample complexity results are summarized in the tables in
Appendix~\ref{appen:tables}.

% \newpage
\section{Learning algorithm and analysis}
\label{sec:optimal_learning}

This section develops the learner underlying the pass@$k$ upper
bounds for $k\geq2$ in Section~\ref{sec:results}. The key idea is to combine
greedy coverage~\citep{nemhauser1978analysis} with multiplicative weights~\citep{arora2012multiplicative}: thresholding turns
bounded rewards into coverage objectives, and selecting $k$ responses
allows an asymmetric weight update that yields the
$\log|\gR|/\log k$ dependence. We first describe and analyze the
online learner, then connect its online-to-batch conversion to the
standard and robust statistical guarantees.
% We combine greedy coverage~\citep{nemhauser1978analysis} with
% multiplicative weights~\citep{arora2012multiplicative} to obtain the
% pass@$k$ upper bounds for $k\geq2$ in~\cref{sec:results}.
% Thresholding turns bounded rewards into coverage objectives, while
% selecting $k$ responses permits asymmetric weight updates yielding
% the $\log|\gR|/\log k$ dependence. We analyze the online learner,
% then derive the standard and robust guarantees by online-to-batch
% conversion.

\subsection{The online learner}

\paragraph{Online game and regret.}
We first consider an online game in which the learner knows the
finite reward class $\gR\subseteq[0,1]^{\gX\times\gY}$ but does not
know which reward is used for evaluation. At each round $t\in[T]$,
it observes a context $x_t\in\gX$ and outputs a tuple
$(\hat y_t^{(1)},\ldots,\hat y_t^{(k)})\in\gY^k$ using the
previously observed demonstrations. It then receives the demonstrated
response $y_t\in\gY$ and updates its state. The learner receives no
evaluation reward feedback. The contexts and demonstrations may be
arbitrary, and the demonstrations need not be optimal; see
Figure~\ref{fig:online-protocol} for the full protocol.

For each $r\in\gR$, we evaluate the learner by its cumulative
pass@$k$ regret relative to the demonstrations,
\begin{equation}
\label{eq:online_regret_criterion}
    \sum_{t=1}^T
    \left[
        r(x_t,y_t)-\max_{i\in[k]}r(x_t,\hat y_t^{(i)})
    \right].
\end{equation}
This criterion compares the learner's best response with the single
demonstrated response at each round. The difference is signed:
rounds on which the learner outperforms the demonstrator offset
rounds on which it falls behind. Our goal is to bound this regret
simultaneously for every $r\in\gR$ and every sequence of contexts
and demonstrations.

\paragraph{Quantization and threshold coverage.}
Our algorithm combines a soft multiplicative weight update with
greedy selection of $k$ responses. To apply these ideas to bounded
rewards, we first express reward values in terms of binary threshold
tests. For binary rewards, a response covers a reward whenever it
scores one, and the best reward in a tuple is determined by whether
any response covers it. Thresholding provides the same coverage
structure for general rewards in $[0,1]$, while retaining information
about their magnitudes.

Fix an integer $L\geq1$ and thresholds $u_j:=j/L$ for $j\in[L]$.
The quantized value of $v\in[0,1]$ is the fraction of thresholds
that it reaches:
\begin{equation}
\label{eq:threshold_quantization}
    q_L(v):=\frac1L\sum_{j=1}^L\ind\{v\geq u_j\}
    =\frac{\lfloor Lv\rfloor}{L}.
\end{equation}
Since $q_L(v)\leq v\leq q_L(v)+1/L$, replacing rewards by their
quantized values changes the per-round regret by at most $1/L$.
Moreover, a tuple covers a threshold precisely when its best reward
reaches that threshold. Thus its covered fraction under $r$ is
$q_L(\max_{i\in[k]}r(x_t,\hat y_t^{(i)}))$, so coverage preserves
the pass@$k$ comparison up to this rounding error. For binary rewards,
$L=1$ already gives an exact representation.

\paragraph{Greedy selection and weight update.}
Algorithm~\ref{alg:passk-general} assigns a positive weight to each
reward-threshold pair in $\gS_L:=\gR\times[L]$, initialized to one.
At context $x_t$, a response $y$ covers
$A_t^y:=\{(r,j)\in\gS_L:r(x_t,y)\geq u_j\}$.
For $A\subseteq\gS_L$, write
$w^{(t)}(A):=\sum_{(r,j)\in A}w^{(t)}(r,j)$ for its total weight.
The learner selects its responses sequentially, choosing at each
step a response with the largest weight of pairs not covered by
earlier selections. This encourages the responses to cover
different pairs. Let $U_t$ be the union covered by the resulting
tuple, and let $D_t:=A_t^{y_t}$ be the set covered by the
demonstration.

After observing $y_t$, the learner multiplies the weight of each
pair in $D_t\setminus U_t$ by $1+k/4$ and of each pair in
$U_t\setminus D_t$ by $3/4$, leaving all other weights unchanged.
The first update gives greater priority to thresholds that only
the demonstrator reaches; the second reduces the priority of
thresholds that only the learner reaches. All weights remain
positive, so a suboptimal demonstration does not eliminate a
candidate reward. The asymmetric factors reflect the coverage
advantage of selecting $k$ responses, as shown in the next
subsection. Ties follow a fixed rule. If $|\gY|<k$, the learner
returns every response and pads the tuple with repetitions,
attaining the optimal reward for every $r\in\gR$.

\begin{algorithm}[ht]
\caption{Greedy coverage with thresholded rewards}
\label{alg:passk-general}
\begin{algorithmic}[1]
\REQUIRE Finite and non-empty reward class $\gR \subseteq [0,1]^{\gX \times \gY}$,
         integers $k\geq2$ and $L,T\geq1$.
\STATE Initialize $w^{(1)}(r,j) = 1$ for all $(r,j) \in \gS_L:=\gR\times[L]$.
\FOR{round $t = 1, 2, \ldots, T$}
  \STATE Receive context $x_t$.
  \STATE For $y\in\gY$, we set, $A_{t}^y = \{(r,j)\in\gS_L\mid r(x_t,y)\geq u_j\},~\text{where }u_j:=\frac{j}{L}.$
  % \[
  %       A_{t}^y = \{(r,j)\in\gS_L\mid r(x_t,y)\geq u_j\},\qquad \text{where }u_j:=\frac{j}{L}.
  % \]
  \STATE \textbf{Greedy top-$k$ selection:} Set $Y_0 = \emptyset$. For $i=1,\ldots,k$:
  \[
    \hat y^{(i)}_t \in \arg\max_{y\in\gY\setminus Y_{i-1}}
    \sum_{(r,j)\in A_t^y \backslash\cup_{z\in Y_{i-1}}A_t^z} w^{(t)}(r,j),
    \quad Y_i \leftarrow Y_{i-1} \cup \{\hat y^{(i)}_t\}.
  \]
  % Equivalently, each $\hat{y}^{(i)}_t$ maximizes the marginal weighted gain
  % among rewards not yet ``satisfied'' by previous selections.
  \STATE Output $(\hat y^{(1)}_t, \ldots, \hat y^{(k)}_t)$.
  \STATE Define $ U_t := \bigcup_{i=1}^kA_t^{\hat{y}_{t}^{(i)}}$.
  % \[
  %       U_t := \bigcup_{i=1}^kA_t^{\hat{y}_{t}^{(i)}}
  % \]
  \STATE Receive demonstration $y_t$.
  \STATE Define $D_t := A_t^{y_t} = \{(r,j): r(x_t,y_t)\ge u_j\}$.

  % \[
  %   D_t := A_t^{y_t} = \{(r,j): r(x_t,y_t)\ge u_j\}.
  % \]
  \STATE \textbf{Weight update:} For each $(r,j) \in \gS_L$:
  % \begin{equation}\label{eq:soft-update}
  %   w^{(t+1)}(r) \;\leftarrow\;
  %   w^{(t)}(r)
  %   \cdot (1+\gamma)^{r(x_t,*) - \max_{i \in [k]} r(x_t,\hat y^{(i)}_t)}
  %   \cdot (1-\gamma)^{r(x_t,*) - r(x_t, y_t)}.
  % \end{equation}
  \[
    w^{(t+1)}(r,j) \gets
    \begin{cases}
      % e^\eta\, w^{(t)}(r,j), & (r,j)\in D_t\setminus U_t,\\
      (1+k/4)\, w^{(t)}(r,j), & (r,j)\in D_t\setminus U_t,\\
      (3/4)\, w^{(t)}(r,j), & (r,j)\in U_t\setminus D_t,\\
      w^{(t)}(r,j), & \text{otherwise.}
    \end{cases}
  \]
\ENDFOR
\end{algorithmic}
\end{algorithm}

\paragraph{Relation to the algorithms of \citet{joshi2026learning}.}
Both the soft update and greedy coverage have precedents in
\citet{joshi2026learning}. Their pass@$1$ algorithm for arbitrary
demonstrators uses a soft multiplicative update on reward functions,
with the factors determined by the learner's and demonstrator's
suboptimality. Their pass@$k$ algorithm for optimal demonstrators
greedily covers rewards under which at least one selected response
is optimal, then eliminates rewards for which the demonstration is
not optimal. Our selection rule uses the same greedy coverage
principle, and our positive weight updates follow the same soft
update approach as their arbitrary-demonstrator algorithm.
The main distinction is that we maintain weights on
reward-threshold pairs: quantization turns bounded reward values
into binary coverage objectives, on which we combine greedy
selection with a soft update based on the learner's performance
relative to the demonstration. This combination permits regret
control under pass@$k$ without assuming optimal demonstrations.

\subsection{Analysis of the online learner}
\label{sec:online_analysis}

The coverage advantage of selecting $k$ responses allows the learner
to increase weights more strongly when it falls behind the
demonstrator than it decreases them when it does better. We show
that the resulting update controls the signed regret, giving the
following guarantee for arbitrary demonstrations.

\begin{theorem}[Online regret]
\label{thm:online-regret}
For integers $k\geq2$ and $L,T\geq1$,
Algorithm~\ref{alg:passk-general} satisfies, for every
$r\in\gR$ and every sequence of contexts and demonstrations,
\[
    \sum_{t=1}^T
    \left[r(x_t,y_t)-\max_{i\in[k]}r(x_t,\hat y_t^{(i)})\right]
    \leq\frac{\log|\gR|}{\log(1+k/4)}+\frac TL
    \leq\frac{2\log|\gR|}{\log k}+\frac TL.
\]
If all rewards are binary and $L=1$, the terms $T/L$ may be omitted.
\end{theorem}

The bound holds simultaneously for all rewards on every realized
sequence. Taking $L\geq T$ makes the rounding contribution at most
one, so the regret remains bounded independently of the horizon.
No independence or demonstrator-optimality assumption is needed.
The proof reduces regret to a bound on threshold weights.

\paragraph{Comparison with pass@$1$.}
The pass@$1$ regret guarantee in Theorem~5 of
\citet{joshi2026learning} depends on the cumulative suboptimality
of the demonstrations. For arbitrary demonstrations, tuning its
learning rate gives worst-case regret of order
$\sqrt{T\log|\gR|}+\log|\gR|$.
Theorem~\ref{thm:online-regret} shows that any $k\geq2$ permits a
bound independent of $T$, apart from the controllable rounding
term, and improves the reward-class dependence to
$\log|\gR|/\log k$. The proof identifies the two ingredients:
greedy coverage permits a larger increase on missed thresholds,
and $\log(1+k/4)\geq\log(4/3)$ ensures that the log-weight change
dominates the signed regret even when the learner outperforms
the demonstrator. The update requires no tuning to the quality
of the demonstrations.

\begin{proof}[Proof of Theorem~\ref{thm:online-regret}]
Fix $r\in\gR$ and a sequence of contexts and demonstrations.
If $|\gY|<k$, the learner returns every response, so its regret
is nonpositive. Assume henceforth that $|\gY|\geq k$.

\paragraph{Reduction to quantized regret.}
To control the regret, first replace rewards by their quantized
values. Define the signed quantized regret at round $t$ by
\begin{equation}
\label{eq:quantized_round_regret}
    \ell_t^r:=q_L\bigl(r(x_t,y_t)\bigr)
      -q_L\!\left(\max_{i\in[k]}r(x_t,\hat y_t^{(i)})\right).
\end{equation}
The rounding property of~\eqref{eq:threshold_quantization} gives
\begin{equation}
\label{eq:online_regret_rounding}
    \sum_{t=1}^T
    \left[r(x_t,y_t)-\max_{i\in[k]}r(x_t,\hat y_t^{(i)})\right]
    \leq\sum_{t=1}^T\ell_t^r+\frac TL.
\end{equation}
It therefore suffices to show that
$\sum_{t=1}^T\ell_t^r\leq\log|\gR|/\log(1+k/4)$.
We relate this sum to the average log-weight associated with $r$,
then use greedy coverage to bound the total weight available to
its thresholds.
\paragraph{From quantized regret to log-weights.}
For the fixed reward $r$, the thresholds covered by either the
demonstration or the learner form an initial segment of
$u_1,\ldots,u_L$. These two segments are therefore nested, and
their shared thresholds cancel in the difference defining
$\ell_t^r$. Thus, $[\ell_t^r]_+$ is the fraction of thresholds
covered only by the demonstration, while $[-\ell_t^r]_+$ is
the fraction covered only by the learner, where
$[v]_+:=\max\{v,0\}$.

The update rule increases the log-weight of each threshold
covered only by the demonstration by $\log(1+k/4)$ and decreases
that of each threshold covered only by the learner by
$\log(4/3)$. All other log-weights remain unchanged.
Averaging these increments over the $L$ thresholds and summing
over rounds gives, since $w^{(1)}(r,j)=1$,
\begin{equation}
\label{eq:log_weight_regret}
\begin{aligned}
    \frac1L\sum_{j=1}^L\log w^{(T+1)}(r,j)
    &=\sum_{t=1}^T\frac1L\sum_{j=1}^L
       \log\frac{w^{(t+1)}(r,j)}{w^{(t)}(r,j)}\\
    &=\sum_{t=1}^T\left\{
       \log(1+k/4)[\ell_t^r]_+
       -\log(4/3)[-\ell_t^r]_+\right\}\\
    &\geq\log(1+k/4)\sum_{t=1}^T\ell_t^r.
\end{aligned}
\end{equation}
For the last inequality, use
$\ell_t^r=[\ell_t^r]_+-[-\ell_t^r]_+$ and
$\log(1+k/4)\geq\log(4/3)$ for $k\geq2$.
In particular, when $\ell_t^r<0$, the average log-weight decreases
by at most $\log(1+k/4)|\ell_t^r|$, so the same inequality
also holds on rounds when the learner outperforms the demonstrator.
It remains to bound the average final log-weight by $\log|\gR|$.
% \paragraph{From quantized regret to log-weights.}
% For a fixed reward, the threshold sets are nested. Thus the
% fractions covered only by the demonstration and only by the
% learner are $[\ell_t^r]_+$ and $[-\ell_t^r]_+$, respectively,
% where $[v]_+:=\max\{v,0\}$. Since the weights start at one,
% the update rule implies
% \begin{equation}
% \label{eq:log_weight_regret}
% \begin{aligned}
%     \frac1L\sum_{j=1}^L\log w^{(T+1)}(r,j)
%     &=\sum_{t=1}^T\left\{
%        \log(1+k/4)[\ell_t^r]_+
%        -\log(4/3)[-\ell_t^r]_+\right\}\\
%     &\geq\log(1+k/4)\sum_{t=1}^T\ell_t^r.
% \end{aligned}
% \end{equation}
% The inequality uses $\log(1+k/4)\geq\log(4/3)$ for $k\geq2$:
% the log-weight change dominates the signed regret multiplied by
% $\log(1+k/4)$, including on rounds when the learner does better.
% It remains to bound the average log-weight by $\log|\gR|$.

\paragraph{Controlling the total weight.}
An upper bound on the total weight
$W_t:=w^{(t)}(\gS_L)$ also bounds the arithmetic mean of the
threshold weights for any fixed reward. The greedy selection rule
gives the comparison
\begin{equation}
\label{eq:greedy_weight_comparison}
    w^{(t)}(U_t\setminus D_t)
    \geq k\,w^{(t)}(D_t\setminus U_t),
\end{equation}
whose proof is given in Appendix~\ref{appen:standard_passk_upper}.
It ensures that the weight removed from pairs covered only by
the learner offsets the weight added to pairs covered only by
the demonstration. Indeed, the update gives
\begin{equation}
\label{eq:online_total_weight}
    W_{t+1}-W_t
    =\frac{k}{4}w^{(t)}(D_t\setminus U_t)
      -\frac14w^{(t)}(U_t\setminus D_t)\leq0.
\end{equation}
Consequently, $W_{T+1}\leq W_1=L|\gR|$. Applying the
arithmetic--geometric mean inequality to the $L$ threshold
weights associated with $r$ now yields
\begin{equation}
\label{eq:online_log_weight_bound}
    \frac1L\sum_{j=1}^L\log w^{(T+1)}(r,j)
    \leq\log\!\left(\frac1L\sum_{j=1}^Lw^{(T+1)}(r,j)\right)
    \leq\log|\gR|.
\end{equation}
Averaging over the thresholds removes the factor $L$ from the
total-weight bound, avoiding a $\log L$ contribution to regret.

\paragraph{Concluding the regret bound.}
Combining~\eqref{eq:log_weight_regret} with
\eqref{eq:online_log_weight_bound} gives the required bound on
$\sum_t\ell_t^r$. Substituting it into
\eqref{eq:online_regret_rounding} proves the first inequality
in the theorem. The identity
$1+k/4-\sqrt{k}=(\sqrt{k}-2)^2/4\geq0$ gives
$\log(1+k/4)\geq\tfrac12\log k$ and hence the second inequality.
For binary rewards with $L=1$, quantization is exact and the
rounding term is absent. This completes the proof.
\end{proof}

\subsection{Online-to-batch conversion}
\label{sec:online_to_batch}

We use a standard online-to-batch conversion~\citep{cesa2004generalization}
to construct a policy from the online learner. Let
$S=((x_1,y_1),\ldots,(x_m,y_m))$ consist of $m\geq1$ i.i.d.\
demonstrations, with $x_t\sim\gD$ and
$y_t\sim\demo(\cdot\mid x_t)$. Fix integers $k\geq2$ and $L\geq1$,
and run Algorithm~\ref{alg:passk-general} for $T=m$ rounds with
fixed tie-breaking. Let $\widehat\pi_t$ be its policy constructed
from the first $t-1$ demonstrations. The learned policy is
\begin{equation}
\label{eq:otob_average}
    \otob(\cdot\mid x)
    :=\frac1m\sum_{t=1}^m\widehat\pi_t(\cdot\mid x),
    \qquad x\in\gX.
\end{equation}
At prediction time, this policy draws one index uniformly from
$[m]$, independently of the training sample and test context,
and generates the entire $k$-response tuple using that snapshot.
Thus the mixture is over complete tuples, whose values average
linearly. Algorithm~\ref{alg:online-to-batch} in the appendix gives
the full procedure. If $|\gY|<k$, returning every response gives
nonpositive excess for every reward, so both statistical guarantees
are immediate. Henceforth assume $|\gY|\geq k$.

\paragraph{Reduction to conditional excess.}
Let $\gF_t:=\sigma((x_s,y_s)_{s\leq t})$, with $\gF_0$ trivial,
and define $g_t^r:=\mathbb E[\ell_t^r\mid\gF_{t-1}]$, using the
quantized regret in~\eqref{eq:quantized_round_regret}.
The snapshot $\widehat\pi_t$ is determined by $\gF_{t-1}$, whereas
$(x_t,y_t)$ is independent of this history. Since $q_L$ is
nondecreasing and commutes with the maximum over the responses,
the usual online-to-batch argument gives
\begin{align}
    g_t^r
    &=V_{q_L\circ r}(\demo)
      -V_{q_L\circ r}^k(\widehat\pi_t),
    \label{eq:gtr_excess}\\
    \gE_r^k(\otob,\demo)
    &\leq\frac1m\sum_{t=1}^m g_t^r+\frac1L.
    \label{eq:excess_value_gtr}
\end{align}
Here the value notation also applies to the quantized reward
$q_L\circ r$. The second line uses the mixture identity and one
rounding error: rounding the demonstrator's value down costs at
most $1/L$, while rounding the learner's value down already gives
an upper bound on the excess.

\paragraph{Controlling the exponential moment.}
The remaining task is to bound $\sum_{t=1}^m g_t^r$ with high
probability. The online regret bound controls the realized sum
$\sum_t\ell_t^r$, but applying Azuma--Hoeffding to
$\sum_t(g_t^r-\ell_t^r)$ introduces a fluctuation term of order
$\sqrt{\log(1/\delta)/m}$ after averaging. This does not give
the desired $1/m$ rate. Instead, we control the exponential
moment of the cumulative conditional excess directly.

\begin{lemma}[Exponential-moment bound]
\label{lem:exponential_moment}
For the learner above, every integer $m\geq1$ and $L\geq1$,
and every reward $r\in\gR$ fixed before sampling,
\begin{equation}
\label{eq:conditional_excess_moment}
    \mathbb E_S\exp\!\left\{\frac13\sum_{t=1}^m g_t^r\right\}
    \leq\exp\!\left\{
       \frac{\log(3/2)}{\log(1+k/4)}\log|\gR|\right\}.
\end{equation}
\end{lemma}

This bound turns the reward-class contribution into a deterministic
shift, leaving an additive $3\log(1/\delta)$ confidence term in
the cumulative conditional excess. Its proof is in
Appendix~\ref{appen:standard_passk_upper}. The problem-specific
ingredient is a one-step exponential inequality for the
asymmetric weight update; iterated conditioning then gives the
bound through a standard exponential-supermartingale argument.

\paragraph{From the moment bound to statistical guarantees.}
Fix $r\in\gR$ before sampling and $\delta\in(0,1)$.
Markov's inequality applied to~\eqref{eq:conditional_excess_moment},
together with the reduction~\eqref{eq:excess_value_gtr}, gives,
with probability at least $1-\delta$,
\begin{equation}
\label{eq:otob_fast_bound}
    \gE_r^k(\otob,\demo)
    \leq\frac{3\log(3/2)\log|\gR|}{m\log(1+k/4)}
       +\frac{3\log(1/\delta)}m+\frac1L.
\end{equation}
Since $\log(1+k/4)\geq\tfrac12\log k$ and
$2\log(3/2)<1$, setting $L=m$ yields the probability bound in
Theorem~\ref{thm:standard_passk_upper}. Requiring its right-hand
side to be at most $\varepsilon$ gives the stated sample
complexity. The guarantee holds for every demonstrator, so it
applies to every suboptimality budget $\Delta\in[0,1]$.

For robust evaluation, apply~\eqref{eq:otob_fast_bound} with
failure probability $\delta/|\gR|$ to each reward and take a
union bound. All these guarantees concern the same learned
policy, because the algorithm depends on neither the evaluation
reward nor the confidence level. With $L=m$, this gives the
probability bound in Theorem~\ref{thm:robust_passk_upper}.
Its numerator is at most $8\log(|\gR|/\delta)+1$, since
$k\geq2$ and $3/\log2+3<8$, proving the stated robust sample
complexity. In particular, simultaneous validity replaces
$\log(1/\delta)$ by $\log(|\gR|/\delta)$, adding a reward-class
term that is not reduced by increasing $k$.

\section{Discussion}

% \cm{Please add a discussion section.}
% the paper characterizes the statistical limits; efficient implementation with large or implicit reward classes and structured response spaces is open.
% \section{Discussion}

For any $k\ge2$, our learner achieves $1/\varepsilon$ sample complexity
even with suboptimal demonstrations. Under standard evaluation, larger
$k$ also reduces the reward-class dependence. Under robust evaluation,
this additional gain can disappear in the high-accuracy regime.
The same learner gives both guarantees.
Our learner assumes a known finite reward class and exact maximization
of weighted coverage, which may be expensive for large reward classes
or response spaces. Whether approximate selection and weight updates
retain the guarantees remains open. Our evaluation compares the best
reward among $k$ responses with the demonstrator's single-response
value. Analyzing an imperfect verifier would quantify the loss from
selecting one candidate.

For $\Delta>0$, our matching robust lower bound requires $\varepsilon$
to be small relative to $\Delta$ and $1/k$. The sharp dependence on $k$
for larger $\varepsilon$ remains open. We allow dependent responses;
the rates under independent sampling from a single policy remain open.

Infinite reward classes require a complexity measure in place of
cardinality. \citet{pour2026learning} use variants of the Littlestone
dimension to characterize realizable online learning with multiple
correct answers and obtain sample bounds for infinite classes.
Their model with one correct response but no feedback on the learner's
correctness is close to our binary-reward setting with optimal
demonstrations. They study single-response prediction with a different
agnostic benchmark. Extending these measures to pass@$k$ with
suboptimal demonstrations remains open.

\section*{Acknowledgement}
C.M.~was partially supported by the National Science Foundation via the CAREER
Award DMS-2443867.

\bibliographystyle{plainnat}
\bibliography{arxiv_main}
\newpage
\appendix

\section{Omitted details and proofs for Section~\ref{sec:results}}
% \subsection{Proof of Theorem~\ref{thm:standard_passk_upper}}\label{appen_sec:high_prob_1}
\subsection{Proof of Theorem~\ref{thm:standard_passk_upper}}
\label{appen:standard_passk_upper}

We supply the greedy coverage calculation used in
Section~\ref{sec:online_analysis} and the exponential-moment
bound used in Section~\ref{sec:online_to_batch}, then conclude
the standard upper bound.
\paragraph{Greedy coverage.}
\begin{proof}[Proof of~\eqref{eq:greedy_weight_comparison}]
Fix a round $t$. If $|\gY|<k$, the learner returns every
response, so $D_t\subseteq U_t$ and the inequality is immediate.
Assume henceforth that $|\gY|\geq k$.

Let $\mathcal U_0:=\varnothing$ and, for $i\in[k]$, define
\[
    \mathcal U_i:=\bigcup_{s=1}^i A_t^{\hat y_t^{(s)}},
    \qquad
    G_i:=\mathcal U_i\setminus\mathcal U_{i-1}.
\]
Thus, $\mathcal U_i$ is the set of pairs covered by the first
$i$ responses, $G_i$ consists of the pairs newly covered at
step $i$, and $\mathcal U_k=U_t$.
The greedy rule gives
\[
    w^{(t)}(G_i)
    \geq w^{(t)}(D_t\setminus\mathcal U_{i-1}).
\]
Indeed, if $y_t$ has not yet been selected, it is an available
candidate whose uncovered weight cannot exceed that of the
greedy choice. If it has already been selected, then
$D_t\subseteq\mathcal U_{i-1}$, so the right-hand side is zero.

To isolate the newly covered weight outside $D_t$, subtract
$w^{(t)}(G_i\cap D_t)$ from both sides. Since
$D_t\setminus\mathcal U_{i-1}$ is the disjoint union of
$G_i\cap D_t$ and $D_t\setminus\mathcal U_i$, we obtain
\[
\begin{aligned}
    w^{(t)}(G_i\setminus D_t)
    &=w^{(t)}(G_i)-w^{(t)}(G_i\cap D_t)\\
    &\geq w^{(t)}(D_t\setminus\mathcal U_{i-1})
          -w^{(t)}(G_i\cap D_t)\\
    &=w^{(t)}(D_t\setminus\mathcal U_i)\\
    &\geq w^{(t)}(D_t\setminus U_t).
\end{aligned}
\]
The last inequality uses $\mathcal U_i\subseteq U_t$:
the weight still uncovered after step $i$ is at least the
weight left uncovered after all $k$ steps.
Finally, the sets $G_i\setminus D_t$ are pairwise disjoint
and have union $U_t\setminus D_t$. Summing over $i$ therefore gives
\[
    w^{(t)}(U_t\setminus D_t)
    =\sum_{i=1}^k w^{(t)}(G_i\setminus D_t)
    \geq k\,w^{(t)}(D_t\setminus U_t),
\]
as claimed.
\end{proof}
% \paragraph{Greedy coverage.}
% \begin{proof}[Proof of~\eqref{eq:greedy_weight_comparison}]
% Fix a round $t$ and assume $|\gY|\geq k$. We compare each
% greedy selection with the demonstrated response, then sum the
% weights newly covered outside $D_t$. Let $\mathcal U_i$ be the
% set of pairs covered by the first $i$ selected responses, so
% $\mathcal U_0=\varnothing$ and $\mathcal U_k=U_t$.
% The greedy choice covers at least as much previously uncovered
% weight as $y_t$ would cover, namely
% $w^{(t)}(D_t\setminus\mathcal U_{i-1})$.
% If $y_t$ has already been selected, this quantity is zero and
% the comparison still holds. Subtract from both sides the weight
% in $D_t$ newly covered by the selected response. This gives
% \[
%     w^{(t)}\bigl(\mathcal U_i\setminus
%         (\mathcal U_{i-1}\cup D_t)\bigr)
%     \geq w^{(t)}(D_t\setminus\mathcal U_i)
%     \geq w^{(t)}(D_t\setminus U_t).
% \]
% The newly covered sets on the left are disjoint as $i$ ranges
% from $1$ to $k$, and their union is $U_t\setminus D_t$.
% Summing proves~\eqref{eq:greedy_weight_comparison}.
% \end{proof}
\paragraph{The exponential-moment bound.}
\begin{proof}[Proof of Lemma~\ref{lem:exponential_moment}]
Fix a reward $r\in\gR$ before sampling. If $|\gY|<k$, the
learner returns every response, so $\ell_t^r\leq0$ and
$g_t^r\leq0$ for every $t$. The left-hand side
of~\eqref{eq:conditional_excess_moment} is then at most one,
and the claim follows. Assume henceforth that $|\gY|\geq k$.

We first express the weight bound in a form suited to controlling
the exponential moment. Set
\[
    \theta_k:=\frac{\log(3/2)}{\log(1+k/4)}\in(0,1],
\]
where the inclusion follows from $k\geq2$, and define the
scaled average log-weight increment
\begin{equation}
\label{eq:scaled_log_increment}
    z_t^r:=\log(3/2)[\ell_t^r]_+
           -\theta_k\log(4/3)[-\ell_t^r]_+.
\end{equation}
By the log-weight identity in~\eqref{eq:log_weight_regret} and
the bound~\eqref{eq:online_log_weight_bound}, every realized
sample satisfies
\[
    \sum_{t=1}^m z_t^r
    =\frac{\theta_k}{L}\sum_{j=1}^L
        \log w^{(m+1)}(r,j)
    \leq\theta_k\log|\gR|.
\]
It remains to relate these increments to the conditional
shortfalls $g_t^r=\mathbb E[\ell_t^r\mid\gF_{t-1}]$.

The scaling is chosen so that $e^{-z_t^r}$ admits a linear
upper bound in $\ell_t^r$. Indeed, for $u\in[0,1]$, convexity
bounds each exponential by the line joining its endpoint values:
\[
\begin{aligned}
    (2/3)^u
    &\leq(1-u)+\frac23u
      =1-\frac u3,\\
    (4/3)^{\theta_k u}
    &\leq(1-u)+u(4/3)^{\theta_k}
      \leq1+\frac u3.
\end{aligned}
\]
The last inequality uses $\theta_k\leq1$.
If $\ell_t^r\geq0$, apply the first bound with $u=\ell_t^r$;
if $\ell_t^r<0$, apply the second with $u=-\ell_t^r$.
Since $|\ell_t^r|\leq1$, both cases give
\begin{equation}
\label{eq:one_step_exponential}
\begin{aligned}
    e^{-z_t^r}
    &\leq1-\frac{\ell_t^r}{3},\\
    \mathbb E[e^{-z_t^r}\mid\gF_{t-1}]
    &\leq1-\frac{g_t^r}{3}
      \leq e^{-g_t^r/3}.
\end{aligned}
\end{equation}
Here the second line follows by conditional expectation and
$1-u\leq e^{-u}$.

To combine these conditional bounds across rounds, define
\[
    K_0:=1,\qquad
    K_t:=\exp\!\left\{
        \sum_{s=1}^t\left(\frac{g_s^r}{3}-z_s^r\right)
    \right\},
    \qquad t\in[m].
\]
This process is adapted and integrable, since
$|g_t^r|\leq1$ and $|z_t^r|\leq\log(3/2)$.
Moreover, $g_t^r$ and $K_{t-1}$ are
$\gF_{t-1}$-measurable, so~\eqref{eq:one_step_exponential} gives
\[
    \mathbb E[K_t\mid\gF_{t-1}]
    =K_{t-1}e^{g_t^r/3}
       \mathbb E[e^{-z_t^r}\mid\gF_{t-1}]
    \leq K_{t-1}.
\]
Thus $(K_t)_{t=0}^m$ is a nonnegative supermartingale, and
iterated expectation yields $\mathbb E_S K_m\leq1$.

Finally, the pathwise bound on $\sum_t z_t^r$ implies
\[
\begin{aligned}
    \mathbb E_S\exp\!\left\{\frac13\sum_{t=1}^m g_t^r\right\}
    &=\mathbb E_S\!\left[
        K_m\exp\!\left\{\sum_{t=1}^m z_t^r\right\}
      \right]\\
    &\leq
      \exp\!\left\{\theta_k\log|\gR|\right\}
      \mathbb E_S K_m\\
    &\leq
      \exp\!\left\{
        \frac{\log(3/2)}{\log(1+k/4)}\log|\gR|
      \right\}.
\end{aligned}
\]
This proves~\eqref{eq:conditional_excess_moment}.
\end{proof}

\paragraph{Conclusion of the standard upper bound.}
\begin{proof}[Proof of Theorem~\ref{thm:standard_passk_upper}]
Let $\gA_k$ run Algorithm~\ref{alg:passk-general} with $T=L=m$
and return the mixture~\eqref{eq:otob_average}. With fixed
tie-breaking, this is a learning rule determined by $S$, and
its prediction-time randomness is already averaged in its value.
Fix $\gD$, $\demo$, and $r^\star\in\gR$ before sampling.
Applying Lemma~\ref{lem:exponential_moment} as in
Section~\ref{sec:online_to_batch}, then setting $L=m$ in
\eqref{eq:otob_fast_bound}, gives
\begin{equation}
\label{eq:standard_passk_upper_tail}
    \mathbb P_S\!\left\{
        \gE_{r^\star}^k(\otob,\demo)
        >\frac{3\log|\gR|}{m\log k}
          +\frac{3\log(1/\delta)+1}{m}
    \right\}\leq\delta.
\end{equation}
Here we used
$\log(3/2)/\log(1+k/4)\leq1/\log k$ for $k\geq2$.
For $0<\varepsilon<1$, the displayed threshold is at most
$\varepsilon$ once $m$ reaches the ceiling below. Since the
bound holds for every demonstrator, it holds whenever
$\Delta_{r^\star}\leq\Delta$ for any $\Delta\in[0,1]$.
Definition~\ref{def:standard_pac_complexity} therefore gives
\begin{equation}
\label{eq:standard_passk_upper_samples}
    m_{\mathrm{std}}^{(k)}
        (\gR,\varepsilon,\delta,\Delta)
    \leq\left\lceil
       \frac{3\log|\gR|/\log k+3\log(1/\delta)+1}{\varepsilon}
    \right\rceil.
\end{equation}
For $\varepsilon\geq1$, the accuracy requirement is automatic
because the excess value is at most one. This completes the proof.
\end{proof}

\subsection{Proof of Theorem~\ref{thm:standard_passk_lower}}
\label{appen:standard_passk_lower}

\begin{proof}
The lower bound has two components. We use the same context space,
response space, and reward class for both. The first construction
distributes the difficulty across several contexts and gives the
dependence on the reward-class size. The second concentrates the
difficulty on one context and gives the dependence on the confidence
level. In both constructions, an unobserved context has an unknown
correct response among $2k$ possibilities, while a pass@$k$ policy can
include at most $k$ distinct responses. We prove the resulting
probability bounds separately and then combine them. Every demonstrator
used below is optimal for its selected evaluation reward.

\paragraph{Common reward class.}
Fix integers $k\geq1$ and $N\geq2k$, and set
\[
    q:=2k,
    \qquad
    d:=\left\lfloor\frac{\log N}{\log q}\right\rfloor\geq1,
    \qquad
    \gX_0:=\{0,1,\ldots,d\},
    \qquad
    \gY:=\{0,1,\ldots,q\}.
\]
We allow $k=1$ so that the same construction also applies in
Appendix~\ref{appen:standard_pass1_lower}.
For each $\theta\in[q]^d$, define
\[
    f_\theta(0):=0,
    \qquad
    f_\theta(j):=\theta_j,\quad j\in[d],
    \qquad
    r_\theta(x,y):=\ind\{y=f_\theta(x)\}.
\]
These give $q^d$ distinct core rewards on $\gX_0\times\gY$.

To obtain exactly $N$ rewards, let $t:=N-q^d$ and introduce
auxiliary contexts $\gX_{\mathrm{aux}}:=\{z_1,\ldots,z_t\}$,
disjoint from $\gX_0$. Set
\[
    \gX:=\gX_0\sqcup\gX_{\mathrm{aux}},
    \qquad
    f_\theta(z_s):=0,\quad s\in[t],
\]
and extend each $r_\theta$ using the same graph-indicator definition.
For each $s\in[t]$, define an additional target function and reward by
\[
    \widetilde f_s(x):=
    \begin{cases}
        1,&x\in[d]\cup\{z_s\},\\
        0,&\text{otherwise},
    \end{cases}
    \qquad
    \widetilde r_s(x,y):=\ind\{y=\widetilde f_s(x)\}.
\]
The common reward class is
\[
    \gR:=\{r_\theta:\theta\in[q]^d\}
        \cup\{\widetilde r_s:s\in[t]\}.
\]
The added reward $\widetilde r_s$ differs from every core reward
at $(z_s,1)$, and two added rewards differ at their corresponding
auxiliary contexts. Hence $|\gR|=q^d+t=N$. If $t=0$, the auxiliary
context set and the added reward family are empty. Every reward is
binary and has optimal value one. All distributions below assign
zero mass to the auxiliary contexts, and the spaces and reward class
depend only on $k$ and $N$.

\paragraph{Two statements to prove.}
Fix $m\geq1$, $0<\delta<1/16$, and an arbitrary, possibly randomized
learner. First, we will find a context distribution $\gD$, a fixed
evaluation reward $r^\star\in\gR$, and a demonstrator $\demo$ optimal
for $r^\star$ such that
\begin{equation}
\label{eq:standard_class_lower_component}
    \mathbb P\!\left\{
        \gE_{r^\star}^k(\widehat\pi,\demo)
        >\frac1{32}\min\!\left\{
            1,\frac{\log N}{m\log(2k)}
        \right\}
    \right\}\geq\frac17>2\delta.
\end{equation}
Second, on the same reward class, we will find a context distribution
and an optimal-demonstrator instance satisfying
\begin{equation}
\label{eq:standard_confidence_lower_component}
    \mathbb P\!\left\{
        \gE_{r^\star}^k(\widehat\pi,\demo)
        >\frac1{32}\min\!\left\{
            1,\frac{\log(1/\delta)}m
        \right\}
    \right\}\geq\frac{8\delta}{3}>2\delta.
\end{equation}
In each statement, $\widehat\pi=\gA_k(\gR,S)$ is the policy returned
from $m$ demonstrations under the selected instance. The instance may
depend on the learner, but its distribution, demonstrator, and evaluation
reward are fixed before the sample is drawn. All probabilities include
the learner's training randomization; prediction-time randomness is
averaged in the policy values. Since each selected demonstrator is
optimal, both statements apply for every budget $\Delta\in[0,1]$.

\paragraph{The pass@$k$ inclusion constraint.}
For a realized output policy, define, for $j\in\gX$ and $a\in\gY$,
\[
    C_j(a):=
    \mathbb P_{\mathbf y\sim\widehat\pi(\cdot\mid j)}
    \{a\in\{y^{(1)},\ldots,y^{(k)}\}\}.
\]
This is the probability that the learner's response tuple includes
$a$. Since a tuple contains at most $k$ distinct responses,
\begin{equation}
\label{eq:standard_inclusion_budget}
    \sum_{a=1}^q C_j(a)
    =\mathbb E_{\mathbf y\sim\widehat\pi(\cdot\mid j)}
        \bigl|[q]\cap\{y^{(1)},\ldots,y^{(k)}\}\bigr|
    \leq k.
\end{equation}
The inequality holds for arbitrarily dependent responses within
the tuple. We use it in both constructions below.

\subsubsection{Dependence on the reward-class size}

\paragraph{Context distribution and demonstrators.}
Let $b:=\min\{m,d\}$ and put
\[
    \gD(j):=\frac1{2m},\quad j\in[b],
    \qquad
    \gD(0):=1-\frac b{2m},
\]
with zero mass elsewhere. Since $b\leq m$, the mass at context
$0$ is at least $1/2$. For each $\theta\in[q]^d$, let
$\thdemo$ return $f_\theta(x)$ deterministically,
and select $r_\theta$ as the evaluation reward.

\paragraph{Admissibility and the target bound.}
The selected reward equals one on every demonstrated response. Thus
$V_{r_\theta}(\thdemo)=V_{r_\theta}^\star=1$,
and $\Delta_{r_\theta}=0$. Define the excess under this fixed reward by
\[
    Z_\theta:=
        \gE_{r_\theta}^k(\widehat\pi,\thdemo).
\]
A tuple succeeds under $r_\theta$ exactly when it contains the
response $f_\theta(j)$ at context $j$. For any distribution supported
on $\gX_0$, this gives
\begin{equation}
\label{eq:standard_passk_graph_loss}
    Z_\theta=\sum_{j=0}^d\gD(j)
        \bigl[1-C_j(f_\theta(j))\bigr].
\end{equation}
Under the present distribution, the contribution from context $0$
is nonnegative, so
\[
    Z_\theta\geq W_\theta,
    \qquad
    W_\theta:=\frac1{2m}\sum_{j=1}^b[1-C_j(\theta_j)],
    \qquad
    0\leq W_\theta\leq\frac b{2m}.
\]
We will show that, for some fixed $\theta$, the quantity $W_\theta$
exceeds $b/(16m)$ with probability at least $1/7$.

\paragraph{Unobserved contexts and the inclusion constraint.}
For the analysis, draw $\theta$ uniformly from $[q]^d$, then draw
the sample from $\gD$ and $\thdemo$, and run
the learner. In this paragraph, $\mathbb P$ and $\mathbb E$ include
this auxiliary draw of $\theta$ as well as the sample and training
randomization. For each $j\in[b]$, let $U_j$ be the event that
context $j$ is absent from the sample. Bernoulli's inequality gives
\[
    \mathbb P(U_j)
    =\left(1-\frac1{2m}\right)^m
    \geq1-m\frac1{2m}=\frac12.
\]
On $U_j$, the sample reveals no information about $\theta_j$.
Because the coordinates of $\theta$ are independent and uniform,
$\theta_j$ remains uniform on $[q]$ conditional on $S$.
The learner's independent randomization supplies no further information,
so the same holds conditional on $S$ and $\widehat\pi$. Consequently,
on $U_j$,
\[
    \mathbb E[1-C_j(\theta_j)\mid S,\widehat\pi]
    =1-\frac1q\sum_{a=1}^q C_j(a)
    \geq1-\frac{k}{q}=\frac12,
\]
where the inequality is \cref{eq:standard_inclusion_budget}.
Writing $W:=W_\theta$, we obtain
\begin{align*}
    \mathbb E W
    &=\frac1{2m}\sum_{j=1}^b
        \mathbb E[1-C_j(\theta_j)]\\
    &\geq\frac1{2m}\sum_{j=1}^b
        \mathbb E\!\left[
            \ind\{U_j\}\,
            \mathbb E[1-C_j(\theta_j)\mid S,\widehat\pi]
        \right]\\
    &\geq\frac1{4m}\sum_{j=1}^b\mathbb P(U_j)
     \geq\frac b{8m}.
\end{align*}
The first inequality keeps only the nonnegative contributions from
unobserved contexts; the second uses the conditional lower bound above.

\paragraph{A probability bound for a fixed instance.}
Since $0\leq W\leq b/(2m)$, splitting its expectation at $b/(16m)$
gives
\[
    \frac b{8m}
    \leq\mathbb E W
    \leq \frac b{16m}
        +\frac{7b}{16m}\mathbb P\!\left\{W>\frac b{16m}\right\}.
\]
It follows that $\mathbb P\{W>b/(16m)\}\geq1/7$.
Let $\mathbb P_\theta$ denote probability over the sample and training
randomization when the parameter is fixed at $\theta$. Averaging over
the auxiliary parameter gives
\[
    \frac1{q^d}\sum_{\theta\in[q]^d}
        \mathbb P_\theta\!\left\{W_\theta>\frac b{16m}\right\}
    =\mathbb P\!\left\{W>\frac b{16m}\right\}
    \geq\frac17.
\]
Therefore at least one fixed $\theta$ satisfies this probability bound.
For that same parameter, $Z_\theta\geq W_\theta$. Moreover,
$d\geq\log N/(2\log q)$, so
\[
    \frac b{16m}
    =\frac1{16}\min\!\left\{1,\frac dm\right\}
    \geq\frac1{32}\min\!\left\{1,\frac{\log N}{m\log q}\right\}.
\]
Taking $r^\star=r_\theta$ and $\demo=\thdemo$
proves \cref{eq:standard_class_lower_component} because $q=2k$.
The selected parameter is fixed before the actual training sample
is drawn.

\subsubsection{Dependence on the confidence level}

\paragraph{Context distribution and demonstrators.}
For this construction, only contexts $0$ and $1$ have positive mass.
Set
\[
    \eta:=\min\!\left\{
        \frac12,\frac{\log(1/(8\delta))}{2m}
    \right\},
    \qquad
    \gD(1):=\eta,
    \qquad
    \gD(0):=1-\eta,
\]
with zero mass elsewhere. For each $\theta\in[q]^d$, select
$r_\theta$ and use $\thdemo$. The demonstrator
returns zero at context $0$ and $\theta_1$ at context $1$.

\paragraph{Admissibility and the target bound.}
Each demonstrator is again optimal for its selected reward, so
$\Delta_{r_\theta}=0$. Write $Z_\theta$ for the excess under
the new context distribution. By \cref{eq:standard_passk_graph_loss},
\[
    Z_\theta=(1-\eta)[1-C_0(0)]+\eta[1-C_1(\theta_1)]
        \geq\eta H_\theta,
    \qquad
    H_\theta:=1-C_1(\theta_1)\in[0,1].
\]
We will find a fixed $\theta$ for which $H_\theta>1/4$ with
probability at least $8\delta/3$. This implies that $Z_\theta>\eta/4$
with at least the same probability.

\paragraph{The event of receiving no informative demonstration.}
Again draw $\Theta$ uniformly from $[q]^d$ for the analysis, and let
$E$ be the event that context $1$ is absent from the sample. On $E$,
every demonstration is the pair $(0,0)$ and reveals nothing about
$\Theta_1$. Since $0<\eta\leq1/2$ and
$\log(1-\eta)\geq-2\eta$,
\[
    \mathbb P(E)=(1-\eta)^m
    \geq e^{-2m\eta}\geq8\delta.
\]
The last inequality follows from
$2m\eta\leq\log(1/(8\delta))$.

\paragraph{A probability bound for a fixed instance.}
On $E$, the label $\theta_1$ is uniform on $[q]$ conditional on
$S,\widehat\pi$. Thus the same inclusion argument as above gives
\[
    \mathbb E[H_\theta\mid S,\widehat\pi]
    =1-\frac1q\sum_{a=1}^q C_1(a)\geq\frac12.
\]
Since $0\leq H_\theta\leq1$, we also have, on $E$,
\[
    \frac12
    \leq\mathbb E[H_\theta\mid S,\widehat\pi]
    \leq\frac14+\frac34
        \mathbb P\{H_\theta>1/4\mid S,\widehat\pi\}.
\]
The conditional probability on the right is therefore at least $1/3$.
Using $Z_\theta\geq\eta H_\theta$ and integrating over $E$, we obtain
\begin{align*}
    \mathbb P\{Z_\theta>\eta/4\}
    &\geq\mathbb P(E\cap\{H_\theta>1/4\})\\
    &=\mathbb E\!\left[
        \ind\{E\}\,
        \mathbb P\{H_\theta>1/4\mid S,\widehat\pi\}
    \right]\\
    &\geq\frac13\mathbb P(E)\geq\frac{8\delta}{3}.
\end{align*}
Averaging over $\theta$ now selects a fixed $\theta$ for which
$\mathbb P_\theta\{Z_\theta>\eta/4\}\geq8\delta/3$.
Because $\delta<1/16$, we have
$\log(1/(8\delta))\geq\tfrac14\log(1/\delta)$, and hence
\[
    \frac\eta4
    =\min\!\left\{\frac18,\frac{\log(1/(8\delta))}{8m}\right\}
    \geq\frac1{32}\min\!\left\{1,\frac{\log(1/\delta)}m\right\}.
\]
This proves \cref{eq:standard_confidence_lower_component} for the
fixed evaluation reward $r_\theta$ and its optimal demonstrator.

\paragraph{Combining the two statements.}
Fix any learner and write
\[
    A:=\frac{\log N}{\log(2k)},
    \qquad
    B:=\log\frac1\delta.
\]
If $A\geq B$, use the first construction; otherwise, use the second.
Both give failure probability greater than $2\delta$, so the selected
optimal-demonstrator instance satisfies
\[
    \mathbb P\!\left\{
        \gE_{r^\star}^k(\widehat\pi,\demo)
        >\frac1{32}\min\!\left\{1,\frac{\max\{A,B\}}m\right\}
    \right\}\geq2\delta.
\]
Since $\max\{A,B\}\geq(A+B)/2$,
\[
    \frac1{32}\min\!\left\{1,\frac{\max\{A,B\}}m\right\}
    \geq\frac1{64}\min\!\left\{1,\frac{A+B}m\right\}.
\]
Lowering the threshold can only increase the probability of exceeding
it. Substituting the definitions of $A$ and $B$ gives
\begin{equation}
\label{eq:standard_passk_lower_tail}
    \mathbb P\!\left\{
        \gE_{r^\star}^k(\widehat\pi,\demo)
        >\frac1{64}\min\!\left\{
            1,\frac1m\left(
                \frac{\log N}{\log(2k)}+\log\frac1\delta
            \right)
        \right\}
    \right\}\geq2\delta.
\end{equation}
Both constructions use the same class of exactly $N$ rewards. For each
learner, the selected distribution, demonstrator, and evaluation reward
are fixed before the sample and training randomization are drawn.
Thus \cref{eq:standard_passk_lower_tail} is a lower bound for standard
evaluation under a fixed reward.

\paragraph{Sample complexity.}
Let $0<\varepsilon\leq1/128$. For every integer
\[
    1\leq m<\frac1{64\varepsilon}\left(
        \frac{\log N}{\log(2k)}+\log\frac1\delta
    \right),
\]
both entries of the minimum in \cref{eq:standard_passk_lower_tail}
are strictly larger than $64\varepsilon$: the first is one, whereas
$64\varepsilon\leq1/2$, and the second exceeds $64\varepsilon$
by the displayed condition on $m$. The excess threshold is therefore
strictly larger than $\varepsilon$, and every learner fails the PAC
requirement at each such sample size.

The upper limit on $m$ above is greater than one, since
$N\geq2k$, $\delta<1/16$, and $\varepsilon\leq1/128$.
In particular, every learner fails at $m=1$. A sample-complexity
threshold of zero is also impossible, because it would require success
at every sample size, including $m=1$.
All selected demonstrators have $\Delta_{r^\star}=0$ and are therefore
admissible for every $\Delta\in[0,1]$. By
\cref{def:standard_pac_complexity}, for every $k\geq1$,
\[
    m_{\mathrm{std}}^{(k)}(\gR,\varepsilon,\delta,\Delta)
    \geq\frac1{64\varepsilon}\left(
        \frac{\log N}{\log(2k)}+\log\frac1\delta
    \right).
\]
Finally, when $k\geq2$, the inequality $\log(2k)\leq2\log k$
implies
\[
    m_{\mathrm{std}}^{(k)}(\gR,\varepsilon,\delta,\Delta)
    \geq\frac1{128\varepsilon}\left(
        \frac{\log N}{\log k}+\log\frac1\delta
    \right).
\]
This proves the theorem with $c=c_0=1/128$ for a reward class of
cardinality $N$ for every $N\geq2k$.
\end{proof}

\subsection{Proof of Theorem~\ref{thm:standard_pass1_lower}}
\label{appen:standard_pass1_lower}
% We first record the information-theoretic tools used here
% and in the robust lower-bound proof. 
For probability laws
$P,Q$ on a finite set, write
\[
    \KL(P\|Q):=\sum_zP(z)\log\frac{P(z)}{Q(z)},
    \qquad
    \TV(P,Q):=\frac12\sum_z|P(z)-Q(z)|,
\]
with the usual conventions at zero, and let $\kl(p,q)$
denote the divergence between $\operatorname{Bern}(p)$
and $\operatorname{Bern}(q)$. 
% Testing probabilities below
% include any randomization used by the test.
\begin{theorem}[Fano's inequality with list decoding;
    {\citealp[Eq.~(11)]{sakai2020generalizations}}]
\label{thm:list_decoding_fano}
Let $N\geq2$ and $1\leq L<N$ be integers.
Let $J$ be uniformly distributed on $[N]$, and let $W$ be
an observation jointly distributed with $J$.
Suppose that $W$ determines a candidate list
$\mathcal L(W)\subseteq[N]$ satisfying
$|\mathcal L(W)|\leq L$ almost surely.
Define the list-decoding error probability by
\[
    e:=\mathbb P\{J\notin\mathcal L(W)\}.
\]
Then
\begin{equation}
\label{eq:standard_pass1_list_fano}
    I(J;W)\geq(1-e)\log\frac NL-\log2,
\end{equation}
where $I(J;W)$ denotes mutual information and all logarithms
are natural. Equivalently,
\[
    \mathbb P\{J\notin\mathcal L(W)\}
    \geq
    1-\frac{I(J;W)+\log2}{\log(N/L)}.
\]
% In particular, if $e\leq\delta$, then
% \[
%     I(J;W)\geq(1-\delta)\log\frac NL-\log2.
% \]
% Any independent randomization used to construct the list
% can be included in $W$.
\end{theorem}
\begin{proof}[Proof of Theorem~\ref{thm:standard_pass1_lower}]
The lower bound has two components. We use the same context space,
response space, and reward class for both. The first component captures
the dependence on the reward-class size, and the second captures the
dependence on the confidence level. Each component contains a fast term,
obtained using optimal demonstrators, and a suboptimality-dependent term.
We use Assouad's lemma and list-decoding Fano for the reward-class
component, and Le Cam's method and the Bretagnolle--Huber inequality for
the confidence component. We prove the corresponding sample-size bounds
and then combine them.

\paragraph{Common reward class.}
Fix an integer $N\geq2$ and set
\[
    d:=\left\lfloor\frac{\log N}{\log2}\right\rfloor\geq1,
    \qquad
    \gY:=\{0,1,2\}.
\]
Take the class of exactly $N$ rewards constructed in
Appendix~\ref{appen:standard_passk_lower} with $k=1$. Write
$\gX_{\mathrm{base}}$ for its context space, including the auxiliary
contexts, and enumerate its rewards as
$r_1^{\mathrm{base}},\ldots,r_N^{\mathrm{base}}$.
On the core contexts $\{0,1,\ldots,d\}$, this class contains the rewards
\[
    r_\theta^{\mathrm{base}}(x,y)
        =\ind\{y=f_\theta(x)\},
    \qquad
    f_\theta(0)=0,
    \qquad
    f_\theta(i)=\theta_i,\quad i\in[d],
\]
indexed by $\theta\in\{1,2\}^d$. Every reward prefers response zero
at context $0$. The auxiliary contexts allow the class to have exactly
$N$ rewards even when $N$ is not a power of two.

Define the extended context space by
\[
    \gX:=\gX_{\mathrm{base}}\sqcup\mathcal Z,
    \qquad
    \mathcal Z:=\{-1,+1\}^N,
\]
where $\sqcup$ denotes disjoint union. For each $j\in[N]$, define
the extended reward $r_j:\gX\times\gY\to\{0,1\}$ by
\[
    r_j(x,y):=
    \begin{cases}
        r_j^{\mathrm{base}}(x,y),&x\in\gX_{\mathrm{base}},\\[2pt]
        \displaystyle
        \frac{1+x_j}{2}\ind\{y=1\}
        +\frac{1-x_j}{2}\ind\{y=2\},&x\in\mathcal Z.
    \end{cases}
\]
Thus, at a new context $x\in\mathcal Z$, reward $r_j$ accepts
response $1$ if $x_j=+1$ and response $2$ if $x_j=-1$.
Let $\mu$ be the uniform distribution on $\mathcal Z$, and set
$\gR:=\{r_j:j\in[N]\}$. We continue to write $r_\theta$ for the
extension of a core reward. The rewards remain distinct on
$\gX_{\mathrm{base}}$, so $|\gR|=N$. Every reward is binary and
has optimal value one. The spaces and reward class depend only on $N$,
and all distributions below assign zero mass to the auxiliary contexts.

\paragraph{Two statements to prove.}
Fix $\Delta\in[0,1]$, $0<\varepsilon,\delta\leq1/64$, and an
arbitrary learner using $m\geq1$ demonstrations. We will show that,
if this learner satisfies the standard PAC requirement for every
admissible instance on $\gR$, then its sample size must satisfy
\begin{equation}
\label{eq:standard_pass1_class_requirement}
    m\geq\frac1{1088}\left[
        \frac{\log N}{\varepsilon}
        +\frac{\Delta}{\varepsilon^2}
            \log\!\left(1+\frac{N\varepsilon^2}
                                {(\Delta+\varepsilon)^2}\right)
    \right]
\end{equation}
and
\begin{equation}
\label{eq:standard_pass1_confidence_requirement}
    m\geq\frac1{256}\left[
        \frac{\log(1/\delta)}{\varepsilon}
        +\frac{\Delta\log(1/\delta)}{\varepsilon^2}
    \right].
\end{equation}
For each construction, admissibility means that the demonstrator's
suboptimality under the selected evaluation reward is at most $\Delta$.
We write $P_\eta$ for the law of one demonstration under an instance
indexed by $\eta$, and $\mathbb P_\eta$ and $\mathbb E_\eta$ for
probability and expectation under $P_\eta^{\otimes m}$, including
the learner's training randomization. Prediction-time randomness is
averaged in the policy values. Each instance and its evaluation reward
are fixed before the sample is drawn.

The next two subsections establish these statements. In the first
subsection, a parameter regime in which the reward-class term is small
will be covered by the two-point estimate proved in the second.

\subsubsection{Dependence on the reward-class size}

\paragraph{Context distribution and optimal demonstrators.}
For the fast term, put
\[
    p:=\frac{8\varepsilon}{d},
    \qquad
    \gD(i):=p,\quad i\in[d],
    \qquad
    \gD(0):=1-dp=1-8\varepsilon,
\]
with zero mass elsewhere. For each $\theta\in\{1,2\}^d$, let
$\thdemo$ return $f_\theta(x)$ deterministically
on the core contexts, and select $r_\theta$ as the evaluation reward.
The demonstrator may be defined arbitrarily outside the support of
$\gD$.

\paragraph{Admissibility and the target bound.}
The selected reward is one on every demonstrated response, so
$\Delta_{r_\theta}=0\leq\Delta$. For the returned policy, define
\[
    Z_\theta:=
        \gE_{r_\theta}^1(\widehat\pi,\thdemo),
    \qquad
    \widehat\theta_i:=
    \begin{cases}
        1,&\widehat\pi(1\mid i)\geq1/2,\\
        2,&\widehat\pi(1\mid i)<1/2.
    \end{cases}
\]
Whenever $\widehat\theta_i\ne\theta_i$, the learner assigns
probability at most $1/2$ to the correct response. Writing
$H(\widehat\theta,\theta):=\sum_{i=1}^d
\ind\{\widehat\theta_i\ne\theta_i\}$, we obtain
\begin{equation}
    Z_\theta
    =(1-dp)\bigl[1-\widehat\pi(0\mid0)\bigr]
        +p\sum_{i=1}^d\bigl[1-\widehat\pi(\theta_i\mid i)\bigr]\geq\frac p2H(\widehat\theta,\theta).
        \label{eq:standard_pass1_assouad_reduction}
\end{equation}
It therefore suffices to find a fixed $\theta$ for which the Hamming
error exceeds $d/4$ with probability greater than $\delta$: on this
event, $Z_\theta>pd/8=\varepsilon$.

\paragraph{Proof using Assouad's lemma.}
Let $\theta^{(i)}$ differ from $\theta$ only at coordinate $i$.
The sample laws $P_\theta^{\otimes m}$ and
$P_{\theta^{(i)}}^{\otimes m}$ agree when context $i$ is absent
and have disjoint supports otherwise. Hence
\[
    \TV(P_\theta^{\otimes m},P_{\theta^{(i)}}^{\otimes m})
    =1-(1-p)^m\leq mp.
\]
Adjoining the learner's independent randomization leaves this distance
unchanged. If $m\leq d/(32\varepsilon)$, then $mp\leq1/4$.
Assouad's lemma in Hamming loss
\citep[Theorem~2.12]{tsybakov2009introduction} gives
\[
    \max_{\theta\in\{1,2\}^d}
        \mathbb E_\theta H(\widehat\theta,\theta)
    \geq\frac d2\left[
        1-\max_{\theta,i}
            \TV(P_\theta^{\otimes m},P_{\theta^{(i)}}^{\otimes m})
    \right]
    \geq\frac{3d}{8}.
\]
Choose a fixed $\theta$ attaining this lower bound and write
$H=H(\widehat\theta,\theta)$. Since $0\leq H\leq d$,
\[
    \frac{3d}{8}
    \leq\mathbb E_\theta H
    \leq\frac d4+\frac{3d}{4}\mathbb P_\theta\{H>d/4\}.
\]
Thus $\mathbb P_\theta\{H>d/4\}\geq1/6>\delta$, and
\cref{eq:standard_pass1_assouad_reduction} implies that the learner
fails the PAC requirement on this fixed instance. Consequently,
every successful sample size must satisfy
\begin{equation}
\label{eq:standard_pass1_fast_class_assouad}
    m\geq\frac{d}{32\varepsilon}
      \geq\frac{\log N}{64\varepsilon},
\end{equation}
where we used $d\geq\log N/(2\log2)$ and $\log2\leq1$.

\paragraph{Context distribution and suboptimal demonstrators.}
For the suboptimality-dependent term, suppose $\Delta\geq8\varepsilon$
and set $h:=2\varepsilon/\Delta\leq1/4$. Put mass $1-\Delta$
on context $0$ and distribute mass $\Delta$ uniformly over $\mathcal Z$:
\[
    \gD(0):=1-\Delta,
    \qquad
    \gD(z):=\Delta\mu(z),\quad z\in\mathcal Z,
\]
with zero mass elsewhere. For each $j\in[N]$, select $r_j$ as the
evaluation reward and define
\[
    \demo^{(j)}(0\mid0)=1,
    \qquad
    \demo^{(j)}(1\mid z)=\frac12+h z_j,
    \qquad
    \demo^{(j)}(2\mid z)=\frac12-h z_j.
\]
Thus the demonstrator slightly favors the response accepted by $r_j$
at every context in $\mathcal Z$.

\paragraph{Admissibility and the target bound.}
Under $r_j$, the demonstrator receives expected reward $1/2+h$ at
every $z\in\mathcal Z$ and reward one at context $0$. Therefore
\[
    \Delta_{r_j}
    =1-\left[(1-\Delta)+\Delta\left(\frac12+h\right)\right]
    =\Delta\left(\frac12-h\right)\leq\Delta.
\]
For the returned policy, write
\[
    Z_j:=\gE_{r_j}^1(\widehat\pi,\demo^{(j)}),
    \qquad
    \widehat f(z):=\widehat\pi(1\mid z)-\widehat\pi(2\mid z),
    \qquad
    c_j:=\mathbb E_{z\sim\mu}[\widehat f(z)z_j].
\]
Here $|\widehat f(z)|\leq1$. The learner's expected reward at $z$ is
\[
    \frac{\widehat\pi(1\mid z)+\widehat\pi(2\mid z)
              +\widehat f(z)z_j}{2}
    \leq\frac{1+\widehat f(z)z_j}{2}.
\]
Since the excess contribution from context $0$ is nonnegative,
\[
    Z_j\geq\Delta\left(h-\frac{c_j}{2}\right).
\]
Because $\varepsilon=\Delta h/2$, the event $Z_j\leq\varepsilon$
requires $c_j\geq h$. Successful learning therefore places the true
reward index in the candidate list
\[
    \mathcal L(\widehat\pi):=\{j\in[N]:c_j\geq h\}.
\]

\paragraph{Proof using Fano's inequality with list decoding.}
A policy may perform well for several rewards, so we bound the number
of indices in this list. The coordinate functions $z\mapsto z_j$
are orthonormal in $L^2(\mu)$. Bessel's inequality gives
\begin{equation}
\label{eq:standard_pass1_list_size}
    h^2|\mathcal L(\widehat\pi)|
    \leq\sum_{j=1}^N c_j^2
    \leq\mathbb E_{z\sim\mu}[\widehat f(z)^2]
    \leq1.
\end{equation}
Every realized output policy therefore yields a list of at most
$\lfloor h^{-2}\rfloor$ candidate rewards.

% We use the following form of list-decoding Fano
% \citep[Eq.~(11)]{sakai2020generalizations}: if $J$ is uniform on $[N]$
% and an observation $W$ determines a list of at most $L<N$ indices,
% then, with $e$ denoting the probability that the list omits $J$,
% \begin{equation}
% \label{eq:standard_pass1_list_fano}
%     I(J;W)\geq(1-e)\log\frac NL-\log2.
% \end{equation}
% We will apply Theorem~\ref{thm:list_decoding_fano} to the
% candidate list $\mathcal L(\widehat\pi)$.
% Indeed, $H(J\mid W)\leq\log2+(1-e)\log L+e\log N$.
% The bound allows lists of size less than $L$.
% We will apply Theorem~\ref{thm:list_decoding_fano} to the
% candidate list $\mathcal L(\widehat\pi)$. By
% \cref{eq:standard_pass1_list_size}, this list contains at most
% $L:=\lfloor h^{-2}\rfloor$ indices for every realized output policy.
% In the regime $Nh^2\geq4$, we have $L<N$, as required by the theorem.
% The list need not contain exactly $L$ indices.

Draw $J$ uniformly from $[N]$ and, conditional on $J=j$, draw
$S\sim P_j^{\otimes m}$. Let $U$ be the learner's independent
training randomization and set $W:=(S,U)$. The output policy and its
candidate list are functions of $W$. Since successful learning requires
the correct index to belong to the list, the PAC requirement gives
\[
    \mathbb P\{J\notin\mathcal L(\widehat\pi)\}
    \leq\frac1N\sum_{j=1}^N\mathbb P_j\{Z_j>\varepsilon\}
    \leq\delta.
\]
Let $P_0$ be the one-demonstration law obtained by replacing $h$ with
zero. The auxiliary-distribution bound on mutual information
\citep[Lemma~4]{scarlett2019introductory} yields
\begin{align*}
    I(J;W)=I(J;S)
    &\leq\frac1N\sum_{j=1}^N
        \KL(P_j^{\otimes m}\|P_0^{\otimes m})\\
    &\overset{(a)}{=}\frac{m\Delta}{N}\sum_{j=1}^N
        \mathbb E_{z\sim\mu}\!\left[
            \kl\left(\frac12+hz_j,\frac12\right)
        \right]\\
    &\overset{(b)}{\leq}4m\Delta h^2.
\end{align*}
Step $(a)$ uses the product and common-context-marginal identities from
\citet[Exercises~14.10 and~14.12]{lattimore2019bandit}: the laws agree at context $0$ and differ
only on $\mathcal Z$. Step $(b)$ uses the identity from~\citet[Appendix~A.2]{de2021bandits}.

Suppose first that $Nh^2\geq4$, and set
$L:=\lfloor h^{-2}\rfloor<N$. The candidate list has at most
$L$ indices, and its error probability
$e:=\mathbb P\{J\notin\mathcal L(\widehat\pi)\}$ satisfies
$e\leq\delta$. Consequently,
\begin{align*}
    4m\Delta h^2
    &\geq I(J;W)\\
    &\overset{(a)}{\geq}(1-e)\log\frac NL-\log2\\
    &\geq(1-\delta)\log\frac NL-\log2\\
    &\geq(1-\delta)\log(Nh^2)-\log2\\
    &\geq\frac14\log(Nh^2).
\end{align*}
Here step $(a)$ is due to Theorem~\ref{thm:list_decoding_fano}.
The next two inequalities use $e\leq\delta$ and $L\leq h^{-2}$,
respectively. The last inequality follows from $Nh^2\geq4$
and $\delta\leq1/64$.
% Suppose first that $Nh^2\geq4$. Then
% $L:=\lfloor h^{-2}\rfloor<N$, and
% \cref{eq:standard_pass1_list_size,eq:standard_pass1_list_fano} give
% \begin{align*}
%     4m\Delta h^2
%     &\geq(1-\delta)\log\frac NL-\log2\\
%     &\geq(1-\delta)\log(Nh^2)-\log2\\
%     &\geq\frac14\log(Nh^2).
% \end{align*}
% The last step uses $Nh^2\geq4$ and $\delta\leq1/64$.
Substituting $h^2=4\varepsilon^2/\Delta^2$, we obtain
\[
    m\geq\frac{\Delta}{64\varepsilon^2}\log(Nh^2)
      \geq\frac{\Delta}{64\varepsilon^2}
          \log\!\left(1+\frac{N\varepsilon^2}
                              {(\Delta+\varepsilon)^2}\right),
\]
where the second inequality follows from
\[
    1+\frac{N\varepsilon^2}{(\Delta+\varepsilon)^2}
    \leq1+\frac{Nh^2}{4}\leq Nh^2.
\]

If $Nh^2<4$, the logarithm in the target bound is at most $\log2$.
The two-point estimate \cref{eq:standard_pass1_slow_component},
proved in the next subsection, then gives
\[
    m\geq\frac{\Delta\log(1/\delta)}{128\varepsilon^2}
      \geq\frac{\Delta}{128\varepsilon^2}
          \log\!\left(1+\frac{N\varepsilon^2}
                              {(\Delta+\varepsilon)^2}\right).
\]
Thus, in both regimes,
\begin{equation}
\label{eq:standard_pass1_class_component}
    m\geq\frac{\Delta}{128\varepsilon^2}
        \log\!\left(1+\frac{N\varepsilon^2}
                            {(\Delta+\varepsilon)^2}\right),
    \qquad \Delta\geq8\varepsilon.
\end{equation}

\paragraph{Conclusion of the reward-class statement.}
When $\Delta\geq8\varepsilon$, combining
\cref{eq:standard_pass1_fast_class_assouad,eq:standard_pass1_class_component}
gives
\[
    m\geq\frac1{256}\left[
        \frac{\log N}{\varepsilon}
        +\frac{\Delta}{\varepsilon^2}
            \log\!\left(1+\frac{N\varepsilon^2}
                                {(\Delta+\varepsilon)^2}\right)
    \right].
\]
When $0\leq\Delta<8\varepsilon$, the suboptimality-dependent term
is bounded by the fast term up to a constant:
\[
    \frac{\Delta}{\varepsilon^2}
        \log\!\left(1+\frac{N\varepsilon^2}
                            {(\Delta+\varepsilon)^2}\right)
    \leq\frac8\varepsilon\log(1+N)
    \leq\frac{16\log N}{\varepsilon}.
\]
Here we used $N\geq2$, so that $\log(1+N)\leq2\log N$.
The bracket in \cref{eq:standard_pass1_class_requirement} is therefore
at most $17\log N/\varepsilon$. Applying
\cref{eq:standard_pass1_fast_class_assouad} and using $17\cdot64=1088$
proves \cref{eq:standard_pass1_class_requirement} in this regime as
well.

\subsubsection{Dependence on the confidence level}

\paragraph{Context distribution and optimal demonstrators.}
Choose two core rewards $r_1^\circ,r_2^\circ$ whose preferred
responses at context $1$ are one and two, respectively. For the fast
term, put mass $w:=4\varepsilon$ on context $1$, mass $1-w$ on
context $0$, and zero mass elsewhere. Under instance $a\in\{1,2\}$,
the demonstrator returns $a$ at context $1$ and zero at context $0$;
the evaluation reward is $r_a^\circ$. Both demonstrators are optimal
for their selected rewards and are admissible for every $\Delta$.

\paragraph{Admissibility and the target bound.}
Write $Z_a:=\gE_{r_a^\circ}^1
(\widehat\pi,\demo^{(a)})$. Decode $\widehat a=1$ if
$\widehat\pi(1\mid1)\geq1/2$ and $\widehat a=2$ otherwise.
A decoding error means that the learner assigns probability at most
$1/2$ to the correct response at context $1$. Consequently,
\[
    Z_a=(1-w)\bigl[1-\widehat\pi(0\mid0)\bigr]
        +w\bigl[1-\widehat\pi(a\mid1)\bigr]
    \geq\frac w2=2\varepsilon
    \quad\text{whenever }\widehat a\ne a.
\]
The PAC requirement therefore bounds each decoding error probability
by $\delta$.

\paragraph{Proof using Le Cam's method.}
The sample laws overlap exactly when no demonstration has context $1$,
an event of probability $(1-w)^m$ under both instances. Thus
\[
    \TV(P_1^{\otimes m},P_2^{\otimes m})=1-(1-w)^m.
\]
Using the Pinsker inequality~\citep{pinsker1964information} and the identity from~\citet[Eq.~(14.13)]{lattimore2019bandit} in the Le Cam's testing bound, gives
\begin{align*}
    \frac12\mathbb P_1\{\widehat a\ne1\}
        +\frac12\mathbb P_2\{\widehat a\ne2\}
    &\geq\frac12\left[1-\TV(P_1^{\otimes m},P_2^{\otimes m})\right]\\
    &=\frac12(1-w)^m
    \geq\frac12e^{-2mw}.
\end{align*}
The last inequality uses $w\leq1/2$. At least one fixed instance has
decoding error at least this average. Since the PAC requirement makes
the average at most $\delta$, we conclude that
\begin{equation}
\label{eq:standard_pass1_fast_confidence_lecam}
    m\geq\frac{\log(1/(2\delta))}{8\varepsilon}
      \geq\frac{\log(1/\delta)}{16\varepsilon}.
\end{equation}

\paragraph{Context distribution and suboptimal demonstrators.}
Suppose now that $\Delta\geq8\varepsilon$ and again set
$h:=2\varepsilon/\Delta\leq1/4$. Put mass $\Delta$ on context $1$,
mass $1-\Delta$ on context $0$, and zero mass elsewhere. For
$\sigma\in\{-1,+1\}$, define
\[
    \sgdemo(0\mid0)=1,
    \qquad
    \sgdemo(1\mid1)=\frac12+\sigma h,
    \qquad
    \sgdemo(2\mid1)=\frac12-\sigma h,
\]
and select $r_+:=r_1^\circ$ and $r_-:=r_2^\circ$ as the
corresponding evaluation rewards. The one-demonstration laws are
\[
    P_\sigma(x,y)=
    \begin{cases}
        1-\Delta,&(x,y)=(0,0),\\
        \Delta(1/2+\sigma h),&(x,y)=(1,1),\\
        \Delta(1/2-\sigma h),&(x,y)=(1,2),\\
        0,&\text{otherwise}.
    \end{cases}
\]
Write $P_+:=P_{+1}$ and $P_-:=P_{-1}$. These laws share the same
context distribution and differ only in the conditional response law
at context $1$.

\paragraph{Admissibility and the target bound.}
Each demonstrator chooses the response accepted by its selected reward
with probability $1/2+h$ at context $1$. Hence
$\Delta_{r_\sigma}=\Delta(1/2-h)\leq\Delta$.
For the returned policy, write
\[
    Z_\sigma:=
        \gE_{r_\sigma}^1(\widehat\pi,\sgdemo),
    \qquad
    \widehat\sigma:=
    \begin{cases}
        +1,&\widehat\pi(1\mid1)\geq1/2,\\
        -1,&\widehat\pi(1\mid1)<1/2.
    \end{cases}
\]
On a decoding error, the learner assigns probability at most $1/2$
to the rewarded response at context $1$, while the demonstrator
assigns probability $1/2+h$. The contribution from context $0$ is
nonnegative, so
\begin{equation}
\label{eq:standard_pass1_testing_reduction}
    \{\widehat\sigma\ne\sigma\}
    \subseteq\{Z_\sigma\geq\Delta h=2\varepsilon\}
    \subseteq\{Z_\sigma>\varepsilon\}.
\end{equation}

\paragraph{Proof using the Bretagnolle--Huber inequality.}
The product identity in~\citet[Exercises~14.10 and~14.12]{lattimore2019bandit} gives
\begin{align*}
    \KL(P_+^{\otimes m}\|P_-^{\otimes m})
    &=m\KL(P_+\|P_-)\\
    &\overset{(a)}{=}m\Delta\,\kl(1/2+h,1/2-h)\\
    &\overset{(b)}{\leq}16m\Delta h^2
      =\frac{64m\varepsilon^2}{\Delta}.
\end{align*}
Step $(a)$ follows by conditioning on the common context marginal;
the conditional response laws agree at context $0$. Step $(b)$ is due the identity from~\citet[Appendix~A.2]{de2021bandits}. These bounds remain valid after adjoining
the learner's independent training randomization.
The Bretagnolle--Huber inequality \citetext{\citet[Theorem~14.2]{lattimore2019bandit};
\cite{tsybakov2009introduction}} yields
\begin{align*}
    \frac12\mathbb P_+\{\widehat\sigma=-1\}
        +\frac12\mathbb P_-\{\widehat\sigma=+1\}
    &\geq\frac14\exp\!\left\{
        -\KL(P_+^{\otimes m}\|P_-^{\otimes m})
    \right\}\\
    &\geq\frac14\exp\!\left\{-\frac{64m\varepsilon^2}{\Delta}\right\}.
\end{align*}
At least one fixed sign has decoding error at least this average.
By \cref{eq:standard_pass1_testing_reduction}, the PAC requirement
bounds each decoding error by $\delta$. Therefore
\[
    \delta\geq\frac14\exp\!\left\{-\frac{64m\varepsilon^2}{\Delta}\right\}.
\]
Using $\log(1/(4\delta))\geq\tfrac12\log(1/\delta)$, we obtain
\begin{equation}
\label{eq:standard_pass1_slow_component}
    m\geq\frac{\Delta\log(1/\delta)}{128\varepsilon^2},
    \qquad \Delta\geq8\varepsilon.
\end{equation}
This is also the estimate used above when $Nh^2<4$.

\paragraph{Conclusion of the confidence statement.}
If $\Delta\geq8\varepsilon$, the bounds
\cref{eq:standard_pass1_fast_confidence_lecam,eq:standard_pass1_slow_component}
imply
\[
    m\geq\frac1{256}\left[
        \frac{\log(1/\delta)}{\varepsilon}
        +\frac{\Delta\log(1/\delta)}{\varepsilon^2}
    \right].
\]
If $0\leq\Delta<8\varepsilon$, the bracket is at most
$9\log(1/\delta)/\varepsilon$, so the fast bound alone gives the
same conclusion, even with the larger constant $1/144$.
This proves \cref{eq:standard_pass1_confidence_requirement} for
every $\Delta\in[0,1]$.

\paragraph{Combining the two statements.}
Let $A$ and $B$ denote the bracketed expressions in
\cref{eq:standard_pass1_class_requirement,eq:standard_pass1_confidence_requirement},
respectively. Both statements apply to every learner that satisfies
the PAC requirement on the same class $\gR$. Thus
$A\leq1088m$ and $B\leq256m$, which give
\[
    m\geq\frac{A+B}{1344}\geq\frac{A+B}{2048}.
\]
The hard instances used to establish the two bounds need not coincide:
the PAC requirement must hold for every admissible instance on this
single reward class.

\paragraph{Sample complexity.}
Zero samples cannot suffice either. In the optimal two-point
construction, the empty-sample laws are identical. Hence the average
decoding error is at least $1/2$, and at least one fixed instance has
$Z_a\geq2\varepsilon$ with probability at least $1/2>\delta$.
Since every successful positive sample size satisfies the preceding
bound, \cref{def:standard_pac_complexity} gives
\begin{equation}
\label{eq:standard_pass1_lower_samples}
\begin{aligned}
    m_{\mathrm{std}}^{(1)}(\gR,\varepsilon,\delta,\Delta)
    \geq\frac1{2048}\Bigg[
        &\frac{\log N+\log(1/\delta)}{\varepsilon}\\
        &+\frac{\Delta}{\varepsilon^2}
        \left\{
            \log\frac1\delta
            +\log\!\left(1+\frac{N\varepsilon^2}
                                {(\Delta+\varepsilon)^2}\right)
        \right\}
    \Bigg].
\end{aligned}
\end{equation}
The constructed class has exactly $N$ rewards for every integer
$N\geq2$. Each testing or averaging argument identifies an admissible
instance and a fixed evaluation reward before the sample is drawn,
as required for standard evaluation. This completes the proof.
\end{proof}

\subsection{Proof of Theorem~\ref{thm:robust_passk_upper}}
\label{appen:robust_passk_upper}

\begin{proof}
Fix $k\geq2$, $m\geq1$, $\delta\in(0,1)$, a context
distribution $\gD$, and an arbitrary demonstrator $\demo$.
Let $\widehat\pi=\gA_k(\gR,S)=\otob$ be the learner
constructed in Appendix~\ref{appen:standard_passk_upper},
with $L=m$. With fixed tie-breaking, this policy is
determined by $S$, and the learning rule does not depend
on the evaluation reward or the confidence level.

For each fixed $r\in\gR$, apply
Theorem~\ref{thm:standard_passk_upper} with failure
probability $\delta/|\gR|$. This gives
\[
    \mathbb P_S\!\left\{
        \gE_r^k(\widehat\pi,\demo)
        >\frac{3\log|\gR|}{m\log k}
         +\frac{3\log(|\gR|/\delta)+1}{m}
    \right\}\leq\frac{\delta}{|\gR|}.
\]
These bounds concern the same learned policy. Since
$\gR$ is finite, a union bound therefore yields
\begin{equation}
\label{eq:robust_passk_upper_tail}
    \mathbb P_S\!\left\{
        \sup_{r\in\gR}\gE_r^k(\widehat\pi,\demo)
        >\frac{3\log|\gR|}{m\log k}
         +\frac{3\log(|\gR|/\delta)+1}{m}
    \right\}
    \leq\sum_{r\in\gR}\frac{\delta}{|\gR|}
    =\delta.
\end{equation}
Thus requiring the guarantee simultaneously over $\gR$
replaces the confidence term $\log(1/\delta)$ by
$\log(|\gR|/\delta)$.

For the sample-complexity conclusion, let
$\varepsilon\in(0,1)$ and $\Delta\in[0,1]$.
The threshold in~\eqref{eq:robust_passk_upper_tail}
is at most $\varepsilon$ whenever
\[
    m\geq\left\lceil
        \frac{3\log|\gR|/\log k
              +3\log(|\gR|/\delta)+1}{\varepsilon}
    \right\rceil.
\]
The same learning rule, setting $L=m$ at each sample
size, satisfies this guarantee for every $\gD$ and
every $\demo$, hence also for all demonstrators with
$\Delta_{\gR}\leq\Delta$.
Definition~\ref{def:robust_pac_complexity} gives
\[
    m_{\mathrm{rob}}^{(k)}
        (\gR,\varepsilon,\delta,\Delta)
    \leq\left\lceil
        \frac{3\log|\gR|/\log k
              +3\log(|\gR|/\delta)+1}{\varepsilon}
    \right\rceil.
\]
Finally, $k\geq2$ and
$\log|\gR|\leq\log(|\gR|/\delta)$ imply
\[
    \frac{3\log|\gR|}{\log k}
    +3\log\frac{|\gR|}{\delta}
    \leq\left(\frac3{\log2}+3\right)
           \log\frac{|\gR|}{\delta}
    \leq8\log\frac{|\gR|}{\delta}.
\]
Substituting this inequality proves
\eqref{eq:robust_passk_upper_samples}.
For $\varepsilon\geq1$, the accuracy requirement is
automatic because
$\sup_{r\in\gR}\gE_r^k(\pi,\demo)\leq1$
for every policy.
\end{proof}

\subsection{Proof of Theorem~\ref{thm:robust_passk_lower}}
\label{appen:robust_passk_lower}

\begin{proof}
The lower bound has two components. We use the same context space,
response space, and reward class for both. The first construction places
the difficulty across several blocks of contexts. A second-moment argument
shows that a block contains an unobserved context with constant probability,
and the pass@$k$ inclusion constraint then forces a nonzero probability of
missing its hidden response. This gives the dependence on $\log N$.
The second construction concentrates the difficulty on a single rare
context and gives the dependence on $\log(1/\delta)$. We prove these two
bounds separately and then combine them.

\paragraph{Common reward class.}
Fix integers $k\geq2$ and $N\geq2k^2$, and set
\[
    q:=2k,
    \qquad
    b:=\left\lfloor\frac{\log N}{\log(2k^2)}\right\rfloor\geq1,
    \qquad
    \gY:=\{0,1,\ldots,q\}.
\]
We first define the rewards on the core context space
$\gX_0:=\{0\}\cup([b]\times[k])$. For each
$v=((j_\ell,a_\ell))_{\ell=1}^b\in([k]\times[q])^b$, let
\[
    r_v(0,y):=1,
    \qquad
    r_v((\ell,j),y):=
    \begin{cases}
        \ind\{y=a_\ell\},&j=j_\ell,\\
        1,&j\ne j_\ell.
    \end{cases}
\]
Thus $r_v$ selects one context in each block and rewards only the
preferred response $a_\ell$ there; it gives reward one at all other
core contexts. These are $(kq)^b=(2k^2)^b$ distinct binary rewards.

To obtain exactly $N$ rewards, set $t:=N-(2k^2)^b$ and introduce
$t$ distinct auxiliary contexts $z_1,\ldots,z_t$ outside $\gX_0$.
Let
\[
    \gX:=\gX_0\cup\{z_1,\ldots,z_t\},
    \qquad
    r_v(z_s,y):=\ind\{y=0\},\quad s\in[t].
\]
Fix one core index $v^0$. For each $s\in[t]$, define an additional
reward $\widetilde r_s$ that agrees with $r_{v^0}$ everywhere except
at $z_s$, where $\widetilde r_s(z_s,y):=\ind\{y=1\}$. Then
\[
    \gR:=\{r_v:v\in([k]\times[q])^b\}
        \cup\{\widetilde r_s:s\in[t]\}
\]
contains exactly $N$ distinct binary rewards. When $t=0$, no auxiliary
contexts or rewards are added. Every reward has optimal value one.
Both constructions below assign zero probability to the auxiliary
contexts. Consequently, each additional reward has the same value as
$r_{v^0}$ under every policy, and maximizing either the suboptimality
or the excess over $\gR$ is equivalent to maximizing over the core
rewards. The spaces and reward class depend only on $k$ and $N$.

\paragraph{Two statements to prove.}
Fix $m\geq1$, $0<\Delta\leq1$, and $0<\delta<1/16$.
We will prove two probability lower bounds on this fixed reward class.
First, for any learner, there exist a context distribution $\gD$ and
a demonstrator $\demo$ satisfying $\Delta_{\gR}\leq\Delta$ such that
\begin{equation}
\label{eq:robust_proof_class_tail}
    \mathbb P\!\left\{
        \sup_{r\in\gR}\gE_r^k(\widehat\pi,\demo)
        >\frac1{96}\min\!\left\{
            \Delta,\frac1k,\frac{\log N}{m}
        \right\}
    \right\}\geq\frac17>2\delta.
\end{equation}
Second, for any learner, there exist a context distribution $\gD$ and
a demonstrator $\demo$ satisfying $\Delta_{\gR}\leq\Delta$ such that
\begin{equation}
\label{eq:robust_proof_confidence_tail}
    \mathbb P\!\left\{
        \sup_{r\in\gR}\gE_r^k(\widehat\pi,\demo)
        >\frac1{32}\min\!\left\{
            \Delta,\frac{\log(1/\delta)}m
        \right\}
    \right\}\geq\frac{8\delta}{3}>2\delta.
\end{equation}
In each statement, $\widehat\pi$ is the policy returned by the learner
from $m$ demonstrations under the selected instance. The instance may
depend on the learner, but is fixed before the sample is drawn.
All probabilities include the learner's training randomization, while
prediction-time randomness is averaged in the policy values.

\paragraph{The pass@$k$ inclusion constraint.}
Fix an arbitrary learner. For a realized output policy $\widehat\pi$,
define the probability that its response tuple contains $a\in[q]$ by
\[
    C_x(a):=
    \mathbb P_{\vy\sim\widehat\pi(\cdot\mid x)}
       \{a\in\{y^{(1)},\ldots,y^{(k)}\}\}.
\]
Since a tuple contains at most $k$ distinct responses,
\begin{equation}
\label{eq:robust_proof_inclusion_budget}
    \sum_{a=1}^q C_x(a)
    =\mathbb E_{\vy\sim\widehat\pi(\cdot\mid x)}
       \bigl|[q]\cap\{y^{(1)},\ldots,y^{(k)}\}\bigr|
    \leq k.
\end{equation}
This constraint holds for arbitrary dependence among the responses.
Both lower bounds use it to show that $k$ responses cannot cover a
uniformly distributed hidden response from $[q]$ with probability
greater than $k/q=1/2$.

\subsubsection{Dependence on the reward-class size}

\paragraph{Context distribution and demonstrators.}
Choose
\[
    p:=\min\!\left\{
        \frac{\Delta}{b},\frac1{2bk},\frac{\log k}{2m}
    \right\},
    \qquad
    \gD((\ell,j)):=p,
    \qquad
    \gD(0):=1-bkp,
\]
and give every auxiliary context zero mass. The condition
$p\leq1/(2bk)$ ensures that $\gD(0)\geq1/2$, so this is a
probability distribution. For each $\theta\in[q]^{bk}$, let
$\thdemo$ return $\theta_{\ell,j}$ at $(\ell,j)$ and zero elsewhere.
The parameter $\theta$ therefore specifies one hidden response at
each core context other than $0$.

\paragraph{Admissibility and the robust excess.}
For each core reward $r_v$, the demonstrator receives reward one at
context $0$ and at the $k-1$ unselected contexts in each block.
At the selected context $(\ell,j_\ell)$, its reward is
$\ind\{a_\ell=\theta_{\ell,j_\ell}\}$. Hence
\begin{align*}
    V_{r_v}^\star-V_{r_v}(\thdemo)
    &=1-(1-bkp)
        -p\sum_{\ell=1}^b\sum_{j=1}^k
            r_v((\ell,j),\theta_{\ell,j})\\
    &=bkp-p\sum_{\ell=1}^b
        \left[(k-1)+\ind\{a_\ell=\theta_{\ell,j_\ell}\}\right]\\
    &=p\sum_{\ell=1}^b
        \ind\{a_\ell\ne\theta_{\ell,j_\ell}\}.
\end{align*}
The last sum is at most $b$, and equality is attained by choosing
a mismatching preferred response in every block. The additional
rewards have the same values as a core reward on the support of
$\gD$, so $\Delta_{\gR}=bp\leq\Delta$. All members of the family
are therefore admissible.

For demonstrations from $\thdemo$, let $\widehat\pi$ be the
learner's returned policy. At a selected context, the expected
pass@$k$ reward is $C_{(\ell,j_\ell)}(a_\ell)$. Thus
\begin{align*}
    V_{r_v}(\thdemo)
    &=(1-bkp)+p\sum_{\ell=1}^b
        \left[(k-1)+\ind\{a_\ell=\theta_{\ell,j_\ell}\}\right],\\
    V_{r_v}^k(\widehat\pi)
    &=(1-bkp)+p\sum_{\ell=1}^b
        \left[(k-1)+C_{(\ell,j_\ell)}(a_\ell)\right].
\end{align*}
Subtracting these expressions gives
\[
    \gE_{r_v}^k(\widehat\pi,\thdemo)
    =p\sum_{\ell=1}^b
        \left[\ind\{a_\ell=\theta_{\ell,j_\ell}\}
              -C_{(\ell,j_\ell)}(a_\ell)\right].
\]
At any selected context, choosing the preferred response to be its
hidden response gives $1-C_{(\ell,j_\ell)}(\theta_{\ell,j_\ell})
\geq0$; every other choice gives a nonpositive contribution.
We may then choose the context with the largest miss probability
independently in each block. Consequently,
\begin{equation}
\label{eq:robust_proof_block_loss}
\begin{aligned}
    Z_\theta
    &:=\sup_{r\in\gR}\gE_r^k(\widehat\pi,\thdemo)\\
    &=p\sum_{\ell=1}^b\max_{j\in[k]}
        \bigl[1-C_{(\ell,j)}(\theta_{\ell,j})\bigr],
    \qquad 0\leq Z_\theta\leq bp.
\end{aligned}
\end{equation}

\paragraph{Unobserved contexts and the inclusion constraint.}
Temporarily draw the coordinates of $\theta$ independently and
uniformly from $[q]$, and then generate the training sample under
$\gD$ and $\thdemo$. Until we select a fixed parameter below,
probabilities and expectations also average over this prior.
Let $U_\ell$ count the contexts in block $\ell$ that do not appear
in the sample. Each context is absent with probability $(1-p)^m$.
Since $p\leq1/(2bk)\leq1/4$ and $2mp\leq\log k$,
\[
    \lambda:=\mathbb E U_\ell
    =k(1-p)^m
    \geq ke^{-2mp}
    \geq1,
\]
where we used $\log(1-p)\geq-2p$ for $p\leq1/2$.
For two distinct contexts, the probability that both are absent is
$(1-2p)^m\leq(1-p)^{2m}$. Expanding the second moment therefore gives
\[
    \mathbb E U_\ell^2
    =k(1-p)^m+k(k-1)(1-2p)^m
    \leq\lambda+\lambda^2.
\]
The second-moment inequality, obtained by applying Paley--Zygmund inequality, now yields
\[
    \mathbb P\{U_\ell>0\}
    \geq\frac{(\mathbb E U_\ell)^2}{\mathbb E U_\ell^2}
    \geq\frac{\lambda}{1+\lambda}
    \geq\frac12.
\]
Thus every block contains an unobserved context with probability
at least $1/2$.

On $\{U_\ell>0\}$, let $J_\ell$ be the smallest index of an
unobserved context in that block. This choice depends only on the
sampled contexts. Conditional on $S$ and $\widehat\pi$, the hidden
response $\theta_{\ell,J_\ell}$ remains uniform on $[q]$: the sample
reveals only the responses at observed contexts, and the learner
depends on $\theta$ only through that sample. Hence, on this event,
\begin{align*}
    \mathbb E\!\left[
        1-C_{(\ell,J_\ell)}(\theta_{\ell,J_\ell})
        \mid S,\widehat\pi
    \right]
    &=1-\frac1q\sum_{a=1}^q C_{(\ell,J_\ell)}(a)\\
    &\geq1-\frac{k}{q}=\frac12,
\end{align*}
by \cref{eq:robust_proof_inclusion_budget}. The maximum in
\cref{eq:robust_proof_block_loss} is at least the miss probability
at this unobserved context. On the complementary event it is
nonnegative. Taking expectations therefore gives
\[
    \mathbb E Z_\theta
    \geq\frac p2\sum_{\ell=1}^b\mathbb P\{U_\ell>0\}
    \geq\frac{bp}{4}.
\]

\paragraph{A probability bound for a fixed instance.}
We next convert this expectation bound into a probability bound.
Since $0\leq Z_\theta\leq bp$, splitting the expectation at $bp/8$
gives
\begin{align*}
    \frac{bp}{4}
    \leq\mathbb E Z_\theta
    &\leq\frac{bp}{8}\mathbb P\{Z_\theta\leq bp/8\}
        +bp\,\mathbb P\{Z_\theta>bp/8\}\\
    &=\frac{bp}{8}
        +\frac{7bp}{8}\mathbb P\{Z_\theta>bp/8\}.
\end{align*}
It follows that $\mathbb P\{Z_\theta>bp/8\}\geq1/7$.
This is an average over the finite prior, so at least one fixed
$\theta\in[q]^{bk}$ satisfies
\[
    \mathbb P_\theta\{Z_\theta>bp/8\}\geq\frac17,
\]
where $\mathbb P_\theta$ now averages only over the training sample
and the learner's training randomization.

To express the threshold in terms of $N$, use
$\lfloor u\rfloor\geq u/2$ for $u\geq1$ and
$\log(2k^2)\leq3\log k$. These imply
\[
    b\log k
    \geq\frac{\log N\,\log k}{2\log(2k^2)}
    \geq\frac{\log N}{6}.
\]
Consequently,
\[
    \frac{bp}{8}
    =\min\!\left\{
        \frac{\Delta}{8},\frac1{16k},\frac{b\log k}{16m}
    \right\}
    \geq\frac1{96}\min\!\left\{
        \Delta,\frac1k,\frac{\log N}{m}
    \right\}.
\]
For the fixed parameter selected above, $Z_\theta$ therefore
exceeds the threshold in \cref{eq:robust_proof_class_tail} with
probability at least $1/7$. This proves the first statement.

\subsubsection{Dependence on the confidence level}

\paragraph{Context distribution and demonstrators.}
Use the same spaces and reward class, and set
\[
    \eta:=\min\!\left\{
        \Delta,\frac12,\frac{\log(1/(8\delta))}{2m}
    \right\},
    \qquad x_\bullet:=(1,1).
\]
Put mass $\eta$ on $x_\bullet$, mass $1-\eta$ on context $0$,
and zero mass on every other context. For each $\theta\in[q]$,
let $\thdemo$ return $\theta$ at $x_\bullet$ and zero elsewhere.
A core reward selecting $x_\bullet$ and a preferred response
distinct from $\theta$ has suboptimality $\eta$, and no reward
has larger suboptimality. Thus $\Delta_{\gR}=\eta\leq\Delta$.
Maximizing the excess as in \cref{eq:robust_proof_block_loss} gives
\[
    Z_\theta
    :=\sup_{r\in\gR}\gE_r^k(\widehat\pi,\thdemo)
    =\eta\bigl[1-C_{x_\bullet}(\theta)\bigr].
\]

\paragraph{The event of receiving no informative demonstration.}
Temporarily draw $\theta$ uniformly from $[q]$, and let $E$ be the
event that $x_\bullet$ is absent from the sample. Since
$0<\eta\leq1/2$ and $2m\eta\leq\log(1/(8\delta))$,
\[
    \mathbb P(E)=(1-\eta)^m
    \geq e^{-2m\eta}
    \geq8\delta.
\]
On $E$, the sample contains only copies of $(0,0)$ and reveals
no information about $\theta$. Thus $\theta$ remains uniform
conditional on $S$ and $\widehat\pi$. By
\cref{eq:robust_proof_inclusion_budget}, the random variable
$H:=1-C_{x_\bullet}(\theta)\in[0,1]$ satisfies, on $E$,
\[
    \mathbb E[H\mid S,\widehat\pi]
    =1-\frac1q\sum_{a=1}^q C_{x_\bullet}(a)
    \geq\frac12.
\]
Splitting this conditional expectation at $1/4$ gives
\[
    \frac12
    \leq\mathbb E[H\mid S,\widehat\pi]
    \leq\frac14+\frac34
        \mathbb P\{H>1/4\mid S,\widehat\pi\}.
\]
The conditional probability is therefore at least $1/3$ on $E$.
Since $Z_\theta=\eta H$, averaging over the sample and training
randomization yields
\[
    \mathbb P\{Z_\theta>\eta/4\}
    \geq\mathbb P\bigl(E\cap\{H>1/4\}\bigr)
    \geq\frac{\mathbb P(E)}3
    \geq\frac{8\delta}{3}.
\]
As before, averaging over the finite prior selects a fixed
$\theta\in[q]$ for which the same probability bound holds over
training alone.

For $\delta<1/16$, we have
$\log(1/(8\delta))\geq\tfrac14\log(1/\delta)$. Since also
$\Delta\leq1$,
\[
    \frac\eta4
    =\min\!\left\{
        \frac{\Delta}{4},\frac18,
        \frac{\log(1/(8\delta))}{8m}
    \right\}
    \geq\frac1{32}\min\!\left\{
        \Delta,\frac{\log(1/\delta)}m
    \right\}.
\]
For the selected fixed parameter, $Z_\theta$ exceeds the threshold
in \cref{eq:robust_proof_confidence_tail} with probability at least
$8\delta/3$. This proves the second statement.

\paragraph{Combining the two statements.}
Fix any learner and write
\[
    A:=\min\!\left\{\Delta,\frac1k,\frac{\log N}{m}\right\},
    \qquad
    B:=\min\!\left\{\Delta,\frac{\log(1/\delta)}m\right\}.
\]
If $A\geq B$, use the first construction. If $B>A$, use the
second construction. Both failure probabilities exceed $2\delta$,
and using the smaller threshold constant $1/96$ in both statements
shows that the selected admissible instance satisfies
\[
    \mathbb P\!\left\{
        Z>\frac1{96}\max\{A,B\}
    \right\}\geq2\delta,
    \qquad
    Z:=\sup_{r\in\gR}\gE_r^k(\widehat\pi,\demo).
\]
Since $\max\{A,B\}\geq(A+B)/2$, lowering the threshold and
substituting the definitions of $A$ and $B$ gives
\begin{equation}
\label{eq:robust_proof_combined_tail}
\begin{aligned}
    \mathbb P\Bigg\{
        Z>\frac1{192}\Bigg[
            &\min\!\left\{
                \Delta,\frac1k,\frac{\log N}{m}
            \right\}\\
            &+\min\!\left\{
                \Delta,\frac{\log(1/\delta)}m
            \right\}
        \Bigg]\Bigg\}\geq2\delta.
\end{aligned}
\end{equation}
Both constructions use the same class of exactly $N$ rewards.
For each learner, the selected context distribution, demonstrator,
and hidden parameter are fixed before the sample and training
randomization are drawn.

\paragraph{Sample complexity.}
Suppose
\[
    0<\varepsilon
    \leq\frac1{384}\min\!\left\{\Delta,\frac1k\right\}.
\]
For every integer
\[
    1\leq m\leq\frac{\log N+\log(1/\delta)}{384\,\varepsilon},
\]
the sum of the two truncated terms in
\cref{eq:robust_proof_combined_tail} satisfies
\begin{align*}
    &\min\!\left\{\Delta,\frac1k,\frac{\log N}{m}\right\}
        +\min\!\left\{\Delta,\frac{\log(1/\delta)}m\right\}\\
    &\qquad\geq\min\!\left\{
        \Delta,\frac1k,\frac{\log N+\log(1/\delta)}m
    \right\}
    \geq384\,\varepsilon.
\end{align*}
The excess threshold is therefore at least $2\varepsilon$, so
every learner fails the robust PAC requirement with probability
at least $2\delta>\delta$ on an admissible instance. This holds
at every such positive integer sample size, including the upper
endpoint when it is an integer. The displayed upper limit exceeds
one under the stated assumptions, so failure at $m=1$ also rules
out a successful learner using zero samples: such a learner could
simply ignore one available demonstration.

Definition~\ref{def:robust_pac_complexity} consequently gives
\[
    m_{\mathrm{rob}}^{(k)}
        (\gR,\varepsilon,\delta,\Delta)
    >\left\lfloor
        \frac{\log N+\log(1/\delta)}{384\,\varepsilon}
    \right\rfloor.
\]
This establishes \cref{eq:robust_passk_lower_samples} for every
$N\geq2k^2$, including the integer endpoint, and completes the proof.
\end{proof}

\subsection{Proof of Theorem~\ref{thm:robust_pass1_lower}}
\label{appen:robust_pass1_minimax}

\begin{proof}
The lower bound has two components. We use the same context space, response space, and reward class for both.
The first construction shows that the learner must obtain information about
many contexts in order to perform well simultaneously for all rewards, which
gives the dependence on $\log N$. The second construction concentrates the
difficulty on a single context and shows that the learner cannot make the
failure probability arbitrarily small without additional samples, which gives
the dependence on $\log(1/\delta)$. We prove these two bounds separately and
then combine them.

\paragraph{Common reward class.}
Fix an integer $N\geq2$ and set
\[
    d:=\left\lfloor\frac{\log N}{\log2}\right\rfloor\geq1,
    \qquad
    \gX:=\{0,1,\ldots,d\},
    \qquad
    \gY:=\{-1,+1\}.
\]
The contexts $1,\ldots,d$ will carry the unknown response biases.
At context $0$, every demonstrator constructed below will always
output $+1$. For each $v\in\{-1,+1\}^d$, define
\[
    f_v(0):=+1,
    \qquad
    f_v(i):=v_i,\quad i\in[d],
    \qquad
    r_v(x,y):=\ind\{y=f_v(x)\}.
\]
These form a core class
$\gR_0:=\{r_v:v\in\{-1,+1\}^d\}$ of $2^d$ distinct binary rewards.
Since $2^d\leq N<2^{d+1}$, choose a subset
$\mathcal V\subseteq\{-1,+1\}^d$ with
$|\mathcal V|=N-2^d$. For every $v\in\mathcal V$, define
\[
    \widetilde r_v(x,y):=
    \begin{cases}
        1,&x=0,\\
        r_v(x,y),&x\in[d],
    \end{cases}
    \qquad
    \gR:=\gR_0\cup\{\widetilde r_v:v\in\mathcal V\}.
\]
The added rewards are distinct from all core rewards because they
equal one at $(0,-1)$, and they are distinct from one another because
their preferred responses on $[d]$ differ. Thus $|\gR|=N$.
When $N=2^d$, take $\mathcal V=\varnothing$.

Every reward in $\gR$ has optimal value one. Moreover, for every
demonstrator $\demo$ used below and every policy $\pi$,
\[
    V_{\widetilde r_v}(\demo)=V_{r_v}(\demo),
    \qquad
    V_{\widetilde r_v}(\pi)\geq V_{r_v}(\pi),
    \qquad v\in\mathcal V.
\]
The first equality holds because the demonstrator always outputs
$+1$ at context $0$; the inequality follows from
$\widetilde r_v\geq r_v$ pointwise. Therefore the added rewards
change neither the maximum demonstrator suboptimality nor the
robust excess: both maxima over $\gR$ equal the corresponding
maxima over $\gR_0$. Both constructions below use this same class
of $N$ rewards on $\gX=\{0,\ldots,d\}$ and $\gY=\{-1,+1\}$.

\paragraph{Two statements to prove.}
Fix $m\geq1$, $0<\Delta\leq1$, and $0<\delta<1/16$.
We will prove two probability lower bounds on this fixed reward class.
First, for any learner, there exist a context distribution $\gD$ and
a demonstrator $\demo$ satisfying $\Delta_{\gR}\leq\Delta$ such that
\begin{equation}
\label{eq:robust_proof_pass1_class_tail}
    \mathbb P\!\left\{
        \sup_{r\in\gR}\gE_r^1(\widehat\pi,\demo)
        >\frac1{128}\min\!\left\{
            \Delta,\sqrt{\frac{\Delta\log N}{m}}
        \right\}
    \right\}\geq\frac18>2\delta.
\end{equation}
Second, for any learner, there exist a context distribution $\gD$ and
a demonstrator $\demo$ satisfying $\Delta_{\gR}\leq\Delta$ such that
\begin{equation}
\label{eq:robust_proof_pass1_confidence_tail}
    \mathbb P\!\left\{
        \sup_{r\in\gR}\gE_r^1(\widehat\pi,\demo)
        >\frac1{16}\min\!\left\{
            \Delta,\sqrt{\frac{\Delta\log(1/\delta)}{m}}
        \right\}
    \right\}\geq2\delta.
\end{equation}
In each statement, $\widehat\pi$ is the policy returned by the learner
from $m$ demonstrations under the selected instance. The instance may
depend on the learner, but is fixed before the sample is drawn.
All probabilities include the learner's training randomization, while
prediction-time randomness is averaged in the policy values.
The next two subsections prove these statements. We then select the
larger logarithmic contribution and convert the resulting probability
bound into the claimed sample-complexity lower bound.

\subsubsection{Dependence on the reward-class size}
Fix an arbitrary learner. We first construct the context distribution and the family of demonstrators. 
We put total mass $\Delta$ equally on the
hard contexts $1,\ldots,d$ and put the remaining mass on context $0$:
\[
    \gD(0)=1-\Delta,
    \qquad
    \gD(i)=\frac{\Delta}{d},\quad i\in[d].
\]
With this context distribution fixed, we construct a family of
demonstrators indexed by $\theta\in\{-1,+1\}^d$.
Let $h\in(0,1/4]$ be a bias magnitude, which we will choose in the
Assouad argument below. For each $\theta$, the policy $\thdemo$
always outputs $+1$ at context $0$. At a hard context $i$, for $t\in\{-1,+1\}$, we set 
\begin{equation}
\label{eq:demo_t_i}
    \thdemo(t\mid i)
    =\frac12+t\theta_i h
    =\frac12+t\left(p_{\theta,i}-\frac12\right),
\end{equation}
where in the last identity we define $p_{\theta,i}:=\thdemo(+1\mid i)$.
Thus $\theta$ specifies the signs of $d$ response biases, while $h$
sets their common magnitude. It is enough to find one member of this
family for which the fixed learner has large robust excess with the
probability required in \cref{eq:robust_proof_pass1_class_tail}.

We first verify that every demonstrator in this family satisfies the required
suboptimality constraint. Every reward has optimal value one,
and at context $0$ the demonstrator receives reward one under every
reward. For each $v\in\{-1,+1\}^d$,
\begin{align*}
    V_{r_v}^\star-V_{r_v}(\thdemo)
    &=1-\sum_{x=0}^d
        \gD(x)\thdemo(f_v(x)\mid x)\\
    &=1-(1-\Delta)
        -\frac{\Delta}{d}\sum_{i=1}^d\thdemo(v_i\mid i)\\
    &=\Delta-\frac{\Delta}{d}
        \sum_{i=1}^d\left(\frac12+v_i\theta_i h\right)\\
    &=\frac{\Delta}{d}\sum_{i=1}^d
        \left(\frac12-v_i\theta_i h\right).
\end{align*}
This is largest when $v=-\theta$, giving
\begin{equation}
\label{eq:robust_proof_pass1_suboptimality}
    \Delta_{\gR}=\Delta(1/2+h)\leq\Delta.
\end{equation}
All members of the family are therefore admissible.

Now we are ready to prove a high probability lower bound on the excess risk using Assouad's lemma. For demonstrations
from $\thdemo$, let $\widehat\pi$ be the learner's returned policy
and define its robust excess by
\[
    Z_\theta:=\sup_{r\in\gR}\gE_r^1(\widehat\pi,\thdemo).
\]
We write $\mathbb P_\theta$ and $\E_\theta$ for probability and
expectation under this demonstrator and the fixed context distribution.
In this notation, our task is to find a fixed $\theta$ for which
$Z_\theta$ exceeds the threshold in
\cref{eq:robust_proof_pass1_class_tail} with probability at least
$1/8$. Such a member of the family gives the context distribution and
admissible demonstrator required by the first statement.

\paragraph{Proof using Assouad's lemma.}
We first relate $Z_\theta$ to the error in estimating the response
probabilities, then decode the signs of these estimates. Assouad's
lemma will show that, for at least one $\theta$, many signs are decoded
incorrectly. Write $q_i:=\widehat\pi(+1\mid i)$ for
$i\in\{0,\ldots,d\}$. Maximizing over the rewards in $\gR$ gives
\begin{align}
    Z_\theta
    &=\sup_{v\in\{-1,+1\}^d}
        \left[V_{r_v}(\thdemo)-V_{r_v}(\widehat\pi)\right]
        \notag\\
    &=\sup_{v\in\{-1,+1\}^d}
        \sum_{x=0}^d\gD(x)
        \left[\thdemo(f_v(x)\mid x)
              -\widehat\pi(f_v(x)\mid x)\right]\notag\\
    &\overset{(a)}{=}(1-\Delta)(1-q_0)
        +\frac{\Delta}{d}
          \sup_{v\in\{-1,+1\}^d}
          \sum_{i=1}^dv_i(p_{\theta,i}-q_i)\notag\\
    &=(1-\Delta)(1-q_0)
        +\frac{\Delta}{d}\sum_{i=1}^d|p_{\theta,i}-q_i|.
        \label{eq:robust_proof_pass1_loss}
\end{align}
For step $(a)$, the context-$0$ term follows from
$f_v(0)=+1$ and $\thdemo(+1\mid0)=1$. For the remaining terms,
\[
    \widehat\pi(t\mid i)
    =\begin{cases}
        q_i,&t=+1,\\
        1-q_i,&t=-1
      \end{cases}
    =\frac12+t\left(q_i-\frac12\right),
    \qquad t\in\{-1,+1\}.
\]
Combining this identity with \cref{eq:demo_t_i} gives
$\thdemo(v_i\mid i)-\widehat\pi(v_i\mid i)
=v_i(p_{\theta,i}-q_i)$. The last equality in
\cref{eq:robust_proof_pass1_loss} follows by choosing each $v_i$
independently to maximize its summand. The first term is nonnegative
because $q_0\leq1$. Hence the estimation error on the hard contexts
gives a lower bound on the robust excess.

To turn this estimation error into a sign-decoding error, let
$\widehat\theta_i=+1$ when $q_i\geq1/2$ and
$\widehat\theta_i=-1$ otherwise. Write
$H(\widehat\theta,\theta):=\sum_{i=1}^d
\ind\{\widehat\theta_i\ne\theta_i\}$ for the number of incorrectly
decoded signs. Every incorrect sign implies
$|p_{\theta,i}-q_i|\geq h$, including a tie at $q_i=1/2$ when
$\theta_i=-1$. Applying this observation to
\cref{eq:robust_proof_pass1_loss} gives
\begin{align}
    Z_\theta
    &\geq\frac{\Delta}{d}\sum_{i=1}^d
        \left|\frac12+h\theta_i-q_i\right|\notag\\
    &=\frac{\Delta}{d}\sum_{i=1}^d
        \left|h\theta_i+\left(\frac12-q_i\right)\right|\notag\\
    &\geq\frac{\Delta h}{d}H(\widehat\theta,\theta).
        \label{eq:robust_proof_pass1_decode}
\end{align}
We will use Assouad's lemma to show that, for some fixed $\theta$,
the Hamming error exceeds $d/8$ with probability at least $1/8$.

For this purpose, we choose $h$ so that neighboring sign vectors
produce training samples that are difficult to distinguish. Let
$P_\theta(x,y):=\gD(x)\thdemo(y\mid x)$ be the law of one
demonstration, and let $\theta^{(i)}$ be obtained from $\theta$
by flipping its $i$-th coordinate. The corresponding demonstrators
differ only at context $i$, which has probability $\Delta/d$.
For $0<h\leq1/4$,
\begin{align}
    \KL(P_\theta^{\otimes m}\parallel
        P_{\theta^{(i)}}^{\otimes m})
    &=m\KL(P_\theta\parallel P_{\theta^{(i)}})\notag\\
    &\overset{(a)}{=}m\sum_{x=0}^d\gD(x)
        \KL\!\left(\thdemo(\cdot\mid x)\parallel
            \demo^{(\theta^{(i)})}(\cdot\mid x)\right)\notag\\
    &=m\gD(i)
        \KL\!\left(\thdemo(\cdot\mid i)\parallel
            \demo^{(\theta^{(i)})}(\cdot\mid i)\right)\notag\\
    &=\frac{m\Delta}{d}\kl(1/2+h,1/2-h)\notag\\
    &\overset{(b)}{\leq}\frac{16m\Delta h^2}{d}.
        \label{eq:robust_proof_pass1_neighbor_KL}
\end{align}
The first equality is the product identity in from~\citet[Exercises~14.10 and~14.12]{lattimore2019bandit}. Step $(a)$ uses the common-context-marginal
identity from the same identity.
% ; these identities are also discussed
% in \citet{lattimore2019bandit}. 
% The next equality holds because
% the conditional laws agree at every context except $i$.
The step $(b)$ follows from~\citet[Appendix~A.2]{de2021bandits}.
Taking
\[
    h:=\min\!\left\{\frac14,\sqrt{\frac{d}{32m\Delta}}\right\}
\]
makes every neighboring divergence at most $1/2$.
These bounds also apply when the learner is randomized, since
adjoining its independent training randomization leaves the KL
divergence unchanged by the result from~\citet[Exercise~14.12]{lattimore2019bandit}.
% \cref{lem:KL-identities}.
Assouad's lemma in Hamming loss
\citetext{\citealp{assouad1983deux};
          \citealp[Theorem~2.12]{tsybakov2009introduction}}
therefore gives
\begin{align}
    \max_{\theta\in\{-1,+1\}^d}
        \E_\theta H(\widehat\theta,\theta)
    &\geq\frac d2\left[
        1-\max_{\theta,i}
          \TV(P_\theta^{\otimes m},P_{\theta^{(i)}}^{\otimes m})
    \right]\notag\\
    &\overset{(a)}{\geq}\frac d2\left[
        1-\sqrt{\frac12\max_{\theta,i}
            \KL(P_\theta^{\otimes m}\parallel
                P_{\theta^{(i)}}^{\otimes m})}
    \right]\geq\frac d4.
        \label{eq:assouad_lem}
\end{align}
Here step $(a)$ is Pinsker's inequality~\citep{pinsker1964information}.

Consequently, there exists a fixed $\theta\in\{-1,+1\}^d$ such
that $\E_\theta H(\widehat\theta,\theta)\geq d/4$. We fix this
parameter for the rest of the argument. Since $0\leq H\leq d$,
writing $H=H(\widehat\theta,\theta)$ gives
\[
    \frac d4
    \leq\E_\theta H
    \leq\frac d8+d\,\mathbb P_\theta\!\left\{H>\frac d8\right\}.
\]
Therefore,
\[
    \mathbb P_\theta\!\left\{H(\widehat\theta,\theta)>
        \frac d8\right\}\geq\frac18.
\]
For this same $\theta$, on the above event,
\cref{eq:robust_proof_pass1_decode} gives $Z_\theta>\Delta h/8$.
Since $d=\lfloor\log N/\log2\rfloor$ and $N\geq2$, we have
$d\geq\log N/(2\log2)\geq\log N/4$, so our choice of $h$
satisfies
\[
    \frac{\Delta h}{8}
    =\min\!\left\{\frac{\Delta}{32},
            \sqrt{\frac{\Delta d}{2048m}}\right\}
    \geq\frac1{128}\min\!\left\{
        \Delta,\sqrt{\frac{\Delta\log N}{m}}
    \right\}.
\]
Thus $Z_\theta$ exceeds the threshold in
\cref{eq:robust_proof_pass1_class_tail} with probability at least
$1/8$. This proves the first statement.

\subsubsection{Dependence on the confidence level}

\paragraph{Context distribution and demonstrators.}
For the confidence bound, we use the same reward class but put all
of the hard-context mass on context $1$: let $\gD(1)=\Delta$,
$\gD(0)=1-\Delta$, and give every other context zero mass.
For each sign $\theta\in\{-1,+1\}$, construct a demonstrator that
outputs $+1$ with probability $1/2+\theta h$ at context $1$ and
always outputs $+1$ elsewhere. Here we take
\[
    h:=\min\!\left\{\frac14,
        \sqrt{\frac{\log(1/(8\delta))}{16m\Delta}}\right\}.
\]
The two demonstrators differ only in which response is slightly more
likely at context $1$.

\paragraph{Admissibility and the target bound.}
Both demonstrators satisfy $\Delta_{\gR}=\Delta(1/2+h)\leq\Delta$,
by the same calculation as in
\cref{eq:robust_proof_pass1_suboptimality}. We again write $Z_\theta$
for the robust excess under the selected demonstrator and the new
context distribution, and write $P_\theta$ for the law of one
demonstration. Explicitly,
\[
    P_\theta(x,y):=\gD(x)\thdemo(y\mid x)
    =\begin{cases}
        1-\Delta,&(x,y)=(0,+1),\\
        \Delta(1/2+\theta h),&(x,y)=(1,+1),\\
        \Delta(1/2-\theta h),&(x,y)=(1,-1),\\
        0,&\text{otherwise}.
      \end{cases}
\]
Write $P_+:=P_{+1}$ and $P_-:=P_{-1}$. These laws share the same
context distribution and differ only in the conditional response law
at context $1$. Since the demonstrations are independent, their joint
law is $P_\theta^{\otimes m}$.
To prove \cref{eq:robust_proof_pass1_confidence_tail}, we must find
at least one fixed sign for which $Z_\theta$ exceeds its stated
threshold with probability at least $2\delta$.

\paragraph{Proof using the Bretagnolle--Huber inequality.}
Fix an arbitrary learner. By the same divergence calculation as in
\cref{eq:robust_proof_pass1_neighbor_KL}, using the results from~\citet[Exercises~14.10 and~14.12]{lattimore2019bandit}
% Step $(a)$ uses the common-context-marginal
% identity from the same identity.
% ; these identities are also discussed
% in \citet{lattimore2019bandit}. 
% The next equality holds because
% the conditional laws agree at every context except $i$.
and ~\citet[Appendix~A.2]{de2021bandits},
% \cref{lem:KL-identities,lem:bernoulli-KL},
\[
    \begin{aligned}
        \KL(P_+^{\otimes m}\parallel P_-^{\otimes m})
        &=m\Delta\,\kl(1/2+h,1/2-h)\\
        &\leq16m\Delta h^2
        \leq\log\frac1{8\delta}.
    \end{aligned}
\]
The last inequality follows from the choice of $h$.
Write $q_i:=\widehat\pi(+1\mid i)$ for $i\in\{0,1\}$ and,
as before, decode the sign as
\[
    \widehat\theta=
    \begin{cases}
        +1,&q_1\geq1/2,\\
        -1,&q_1<1/2.
    \end{cases}
\]
Here $\mathbb P_\theta$ denotes probability under sample law
$P_\theta^{\otimes m}$, including the learner's training
randomization. Write $\mathbb P_+:=\mathbb P_{+1}$ and
$\mathbb P_-:=\mathbb P_{-1}$. The Bretagnolle--Huber inequality~\citetext{\citet[Theorem~14.2]{lattimore2019bandit};
\cite{tsybakov2009introduction}}, gives
\[
    \begin{aligned}
        \frac12\mathbb P_+\{\widehat\theta=-1\}
        +\frac12\mathbb P_-\{\widehat\theta=+1\}
        &\geq\frac14\exp\!\left\{
            -\KL(P_+^{\otimes m}\parallel P_-^{\otimes m})
        \right\}\\
        &\geq2\delta.
    \end{aligned}
\]
It follows that, for at least one fixed sign $\theta$, the decoding
error probability is at least $2\delta$. We choose this sign before
the training sample is drawn. The same maximization over rewards as
in \cref{eq:robust_proof_pass1_loss} gives
\[
    Z_\theta=(1-\Delta)(1-q_0)
        +\Delta\left|\frac12+\theta h-q_1\right|
    \geq\Delta\left|\frac12+\theta h-q_1\right|.
\]
If $\theta=+1$, a decoding error means $q_1<1/2$, whereas
if $\theta=-1$, it means $q_1\geq1/2$. In either case,
$|1/2+\theta h-q_1|\geq h$. Thus every decoding error implies
$Z_\theta\geq\Delta h$, and hence
\[
    \mathbb P_\theta\{Z_\theta\geq\Delta h\}
    \geq\mathbb P_\theta\{\widehat\theta\ne\theta\}
    \geq2\delta.
\]
Since $\delta<1/16$,
$\log(1/(8\delta))\geq\tfrac14\log(1/\delta)$, and the chosen $h$
satisfies
\[
    \Delta h
    \geq\min\!\left\{\frac{\Delta}{4},
        \sqrt{\frac{\Delta\log(1/\delta)}{64m}}\right\}
    >\frac1{16}\min\!\left\{
        \Delta,\sqrt{\frac{\Delta\log(1/\delta)}{m}}
    \right\}.
\]
Every decoding error therefore gives robust excess strictly above
the threshold in \cref{eq:robust_proof_pass1_confidence_tail}.
For the chosen sign, this happens with probability at least $2\delta$,
which proves the second statement.

\paragraph{Combining the two statements.}
Fix any learner and write
\[
    A:=\log N,
    \qquad
    B:=\log(1/\delta).
\]
If $A\geq B$, use the first construction, which gives a bound
involving $A$ with probability at least $1/8>2\delta$.
If $B>A$, use the second construction, which gives a bound
involving $B$ with probability at least $2\delta$.
Using the smaller threshold constant $1/128$ in both
\cref{eq:robust_proof_pass1_class_tail,eq:robust_proof_pass1_confidence_tail},
the selected admissible instance satisfies
\[
    \mathbb P\!\left\{
        Z>\frac1{128}\min\!\left\{
            \Delta,\sqrt{\frac{\Delta\max\{A,B\}}{m}}
        \right\}
    \right\}\geq2\delta,
\]
where $Z:=\sup_{r\in\gR}\gE_r^1(\widehat\pi,\demo)$.
Since $\max\{A,B\}\geq(A+B)/2$,
\[
    \begin{aligned}
        \frac1{128}\min\!\left\{
            \Delta,\sqrt{\frac{\Delta\max\{A,B\}}{m}}
        \right\}
        &\geq\frac1{128\sqrt2}\min\!\left\{
            \Delta,\sqrt{\frac{\Delta(A+B)}{m}}
        \right\}\\
        &\geq\frac1{256}\min\!\left\{
            \Delta,\sqrt{\frac{\Delta(A+B)}{m}}
        \right\}.
    \end{aligned}
\]
Lowering the threshold can only increase the probability of exceeding
it. Substituting the definitions of $A$ and $B$, we obtain
\begin{equation}
\label{eq:robust_proof_pass1_combined_tail}
    \mathbb P\!\left\{
        \sup_{r\in\gR}\gE_r^1(\widehat\pi,\demo)
        >\frac1{256}\min\!\left\{
            \Delta,
            \sqrt{\frac{\Delta[\log N+\log(1/\delta)]}{m}}
        \right\}
    \right\}\geq2\delta.
\end{equation}
Both constructions use the same class of $N=|\gR|$ rewards on
$\gX=\{0,\ldots,d\}$. For each learner, the selected instance and
its hidden parameter are fixed before the sample and training
randomization are drawn.

Suppose $0<\varepsilon\leq\Delta/256$. For every integer
\[
    1\leq m\leq
    \frac{\Delta[\log N+\log(1/\delta)]}
         {262144\,\varepsilon^2},
\]
the two entries of the minimum in
\cref{eq:robust_proof_pass1_combined_tail} are at least
$256\varepsilon$ and $512\varepsilon$, respectively.
The excess threshold is therefore at least
\[
    \frac1{256}\min\{256\varepsilon,512\varepsilon\}
    =\varepsilon,
\]
and every learner fails the robust PAC requirement at each such
sample size, including the upper endpoint when it is an integer.

Zero samples also cannot suffice. In the two-sign construction,
take $h=1/4$. Without data, the decoded sign is independent of a
uniform hidden sign, so its average error probability is $1/2$.
For at least one fixed sign, the robust excess is therefore at least
$\Delta/4>\varepsilon$ with probability at least $1/2$.
Combining these failures with \cref{def:robust_pac_complexity} gives
\[
    m_{\mathrm{rob}}^{(1)}
        (\gR,\varepsilon,\delta,\Delta)
    >\left\lfloor
        \frac{\Delta[\log N+\log(1/\delta)]}
             {262144\,\varepsilon^2}
    \right\rfloor,
\]
which establishes the lower bound in
\cref{eq:robust_pass1_lower_samples} for the constructed class
of cardinality $N=|\gR|$ for every $N\geq2$, including the integer endpoint.
This completes the proof.
\end{proof}

\paragraph{Upper bound and minimax rate.}
Theorem~6 and Appendix~D.2 of~\citet{joshi2026learning}
give a single policy whose guarantee holds simultaneously
for all rewards with demonstrator suboptimality at most
$\Delta$. Thus, whenever $\Delta_{\gR}\leq\Delta$,
with probability at least $1-\delta$,
\[
    \sup_{r\in\gR}\gE_r^1(\widehat\pi,\demo)
    \leq C\left(
        \frac{\log(|\gR|/\delta)}{m}
        +\sqrt{\frac{\Delta\log(|\gR|/\delta)}{m}}
    \right)
\]
for a universal constant $C$. Making each term at most
$\varepsilon/2$ gives
\[
    m_{\mathrm{rob}}^{(1)}
        (\gR,\varepsilon,\delta,\Delta)
    \lesssim
    \frac{\log(|\gR|/\delta)}{\varepsilon}
    +\frac{\Delta\log(|\gR|/\delta)}{\varepsilon^2}.
\]
For $0<\varepsilon\leq\Delta/256$, the second term
dominates. The upper bound holds uniformly over reward
classes, while the construction above supplies a matching
worst-case lower bound of order
$\Delta\log(|\gR|/\delta)/\varepsilon^2$.

\subsection{Uniformly optimal demonstrators}
\label{appen:sec_unif_opt_demon}
For completeness, suppose $\Delta_{\gR}=0$. For every
$r\in\gR$,
\[
    \mathbb E_{x\sim\gD,\,y\sim\demo(\cdot\mid x)}
    \left[\sup_z r(x,z)-r(x,y)\right]=0.
\]
The integrands are nonnegative and $\gR$ is finite.
Hence, for $\gD$-almost every $x$, the demonstrator
assigns probability one to responses maximizing every
reward simultaneously. In particular,
\[
    \max_{y\in\gY}\sum_{r\in\gR}r(x,y)
    =\sum_{r\in\gR}\sup_{y\in\gY}r(x,y)
    \qquad\text{for $\gD$-almost every $x$.}
\]
Using the same selection convention as in the algorithms,
choose $y_0(x)$ maximizing the sum of the known rewards
whenever the maximum exists, and an arbitrary response
otherwise. At almost every context,
\[
    \sum_{r\in\gR}
       \bigl[\sup_y r(x,y)-r(x,y_0(x))\bigr]=0.
\]
Each summand is nonnegative, so $y_0(x)$ maximizes every
reward. For any $k\geq1$, the policy $\pi_0$ returning
$k$ copies of $y_0(x)$ therefore satisfies
\[
    V_r^k(\pi_0)=V_r^\star=V_r(\demo)
    \quad(r\in\gR),
    \qquad
    m_{\mathrm{rob}}^{(k)}
       (\gR,\varepsilon,\delta,0)=0.
\]
This policy depends only on the known reward class and
uses no demonstrations.
\newpage
\section{Online learning protocol and online-to-batch conversion}
\begin{figure}[h]
\centering
\setlength{\fboxsep}{8pt}
\fbox{%
\begin{minipage}{\dimexpr\linewidth-2\fboxsep-2\fboxrule\relax}
\small
\textbf{Online learning from demonstrations under pass@$k$ evaluation}

\medskip
The learner knows $\gR$ but not the evaluation reward $r^\star\in\gR$.

\medskip
For each round $t\in[T]$:
\begin{enumerate}
    \setlength{\itemsep}{4pt}
    \setlength{\parskip}{0pt}

    \item \textbf{Observe the context.} Receive $x_t\in\gX$.

    \item \textbf{Produce responses.}
    Using $x_t$ and past demonstrations $\{(x_s,y_s)\}_{s<t}$, output
    $(\hat y_t^{(1)},\ldots,\hat y_t^{(k)})\in\gY^k$
    before observing $y_t$.

    \item \textbf{Evaluate.} The tuple receives reward
    \[
        \max_{i\in[k]}r^\star(x_t,\hat y_t^{(i)}),
    \]
    which is hidden from the learner. For binary rewards, a mistake
    means all $k$ responses score zero.

    \item \textbf{Observe and update.}
    Receive $y_t\in\gY$ and update using $(x_t,y_t)$.
    The demonstration need not be optimal.
\end{enumerate}
\end{minipage}%
}
\caption{Online protocol under pass@$k$ evaluation,
adapted from~\cite{joshi2026learning}.}
\label{fig:online-protocol}
\end{figure}
\begin{algorithm}[h]
\caption{Online-to-batch conversion for pass@$k$}
\label{alg:online-to-batch}
\begin{algorithmic}[1]
\REQUIRE Finite reward class $\gR$, sample
$S=((x_t,y_t))_{t=1}^m\sim(\gD\times\demo)^{\otimes m}$,
response budget $k\geq2$, and discretization level $L\in\mathbb N$.
\ENSURE Policy $\otob:\gX\rightarrow\Delta(\gY^k)$.

\STATE Initialize Algorithm~\ref{alg:passk-general} with $\gR,k,L$.
\FOR{$t=1,\ldots,m$}
    \STATE Save a fixed copy $\widehat\pi_t$ of the current policy,
    based only on $\{(x_s,y_s)\}_{s<t}$.
    \STATE Produce $k$ responses at $x_t$.
    \STATE Observe $y_t$ and update by Algorithm~\ref{alg:passk-general}.
\ENDFOR

\STATE Define the uniform mixture
\[
    \otob(\cdot\mid x)
    :=\frac1m\sum_{t=1}^m\widehat\pi_t(\cdot\mid x),
    \qquad x\in\gX.
\]
\STATE \textbf{Prediction at $x$:}
Draw $\tau\sim\operatorname{Unif}([m])$, independently of $S$ and $x$,
and return an entire response tuple sampled from
$\widehat\pi_\tau(\cdot\mid x)$.
\end{algorithmic}
\end{algorithm}
% \begin{algorithm}[h]
% \caption{Online-to-batch conversion for pass@$k$}
% \label{alg:online-to-batch}
% \begin{algorithmic}[1]
% \REQUIRE Finite reward class $\gR$,
% training sample
% $S=((x_1,y_1),\ldots,(x_m,y_m))
% \sim (\gD\times \demo)^{\otimes m}$,
% response budget $k\geq2$,
% discretization level $L\in\mathbb N$.
% \ENSURE A policy
% $\otob:\gX\rightarrow\Delta(\gY^k)$.

% \STATE Initialize Algorithm~\ref{alg:passk-general}
% with reward class $\gR$ and parameters $k,L$.

% \FOR{$t=1,\ldots,m$}
%     \STATE Before processing $(x_t,y_t)$, save a fixed
%     copy of the current online policy
%     \[
%         \widehat{\pi}_t:
%         \gX\rightarrow\Delta(\gY^k).
%     \]
%     This policy depends only on the demonstrations
%     $\{(x_s,y_s)\}_{s<t}$.

%     \STATE Provide $x_t$ to the online learner and
%     obtain its response tuple
%     $(\hat y_t^{(1)},\ldots,\hat y_t^{(k)})$.

%     \STATE Reveal $y_t$ and update the online learner
%     according to Algorithm~\ref{alg:passk-general}.
% \ENDFOR

% \STATE Define the learned policy as the uniform mixture
% \[
%     \otob(\cdot\mid x)
%     :=
%     \frac1m\sum_{t=1}^m
%     \widehat{\pi}_t(\cdot\mid x),
%     \qquad x\in\gX.
% \]

% \STATE \textbf{Prediction at a new context $x$:}
% \STATE Draw $\tau\sim\operatorname{Unif}([m])$,
% independently of the training sample and $x$.
% \STATE Sample the entire response tuple
% \[
%     (\hat y^{(1)},\ldots,\hat y^{(k)})
%     \sim\widehat{\pi}_{\tau}(\cdot\mid x).
% \]
% \STATE \textbf{Output}
% $(\hat y^{(1)},\ldots,\hat y^{(k)})$.
% \end{algorithmic}
% \end{algorithm}
\newpage
\section{Tables}
\label{appen:tables}
% Arbitrary demonstrators
\begin{table}[ht]
\centering
\caption{
Minimax PAC sample complexities for arbitrary demonstrators
($\Delta=1$), with universal constants implicit in $\Theta$.
Upper bounds hold for every finite reward class, and matching
lower bounds hold for suitable classes of the indicated cardinality.
Assume $|\gR|\geq2$, $0<\delta\leq1/64$, and
$0<\varepsilon\leq c$, where $c>0$ is a sufficiently small
universal constant. For standard pass@$1$, assume additionally
$|\gR|\geq(1+1/\varepsilon)^4$ or $\delta\leq1/|\gR|$.
For pass@$k$, standard evaluation requires $|\gR|\geq2k$;
robust evaluation requires $|\gR|\geq2k^2$ and
$\varepsilon\leq c/k$.
}
% \label{tab:intro-rates}
\label{tab:sample_complexity_arbitrary}
\small
\setlength{\tabcolsep}{5pt}
\renewcommand{\arraystretch}{1.6}
\begin{tabular}{@{}lcc@{}}
\toprule
& \textbf{Standard evaluation}
& \textbf{Robust evaluation} \\
\midrule
pass@$1$
&
$\displaystyle
\Theta\!\left(
    \frac{\log(|\gR|/\delta)}{\varepsilon^2}
\right)$
&
$\displaystyle
\Theta\!\left(
    \frac{\log(|\gR|/\delta)}{\varepsilon^2}
\right)$
\\[3mm]
pass@$k$, $k\geq2$
&
$\displaystyle
\Theta\!\left[
    \frac1\varepsilon
    \left(
        \frac{\log|\gR|}{\log k}
        +\log\frac1\delta
    \right)
\right]$
&
$\displaystyle
\Theta\!\left(
    \frac{\log(|\gR|/\delta)}{\varepsilon}
\right)$
\\
\bottomrule
\end{tabular}
\end{table}
% Positive suboptimality budgets
\begin{table}[ht]
\centering
\caption{
Minimax PAC sample complexities for $0<\Delta\leq1$.
The budget constrains $\Delta_{r^\star}\leq\Delta$ under
standard evaluation and $\Delta_{\gR}\leq\Delta$ under
robust evaluation. Assume $|\gR|\geq2$,
$0<\delta\leq1/64$, and $0<\varepsilon\leq c$, where
$c>0$ is a sufficiently small universal constant.
For standard pass@$1$, assume additionally
$|\gR|\geq(1+\Delta/\varepsilon)^4$ or
$\delta\leq1/|\gR|$; robust pass@$1$ requires
$\varepsilon\leq c\Delta$.
For pass@$k$, standard evaluation requires $|\gR|\geq2k$;
robust evaluation requires $|\gR|\geq2k^2$ and
$\varepsilon\leq c\min\{\Delta,1/k\}$.
The $\Theta$ entries describe matching worst-case bounds
over reward classes, with universal constants.
}
\label{tab:sample_complexity_budget}
\small
\setlength{\tabcolsep}{5pt}
\renewcommand{\arraystretch}{1.6}
\begin{tabular}{@{}lcc@{}}
\toprule
& \textbf{Standard evaluation}
& \textbf{Robust evaluation} \\
\midrule
pass@$1$
&
$\displaystyle
\Theta\!\left(
    \frac{(\Delta+\varepsilon)\log(|\gR|/\delta)}
         {\varepsilon^2}
\right)$
&
$\displaystyle
\Theta\!\left(
    \frac{\Delta\log(|\gR|/\delta)}{\varepsilon^2}
\right)$
\\[3mm]
pass@$k$, $k\geq2$
&
$\displaystyle
\Theta\!\left[
    \frac1\varepsilon
    \left(
        \frac{\log|\gR|}{\log k}
        +\log\frac1\delta
    \right)
\right]$
&
$\displaystyle
\Theta\!\left(
    \frac{\log(|\gR|/\delta)}{\varepsilon}
\right)$
\\
\bottomrule
\end{tabular}
\end{table}
% Optimal demonstrators
\begin{table}[!ht]
\centering
\caption{
Minimax PAC sample complexities for optimal demonstrators
($\Delta=0$). Under standard evaluation, optimality means
$\Delta_{r^\star}=0$ for the fixed evaluation reward;
under robust evaluation, it means $\Delta_{\gR}=0$,
that is, optimality for every reward in the class.
The standard entries assume $|\gR|\geq2$,
$0<\delta\leq1/64$, and $0<\varepsilon\leq c$ for a
sufficiently small universal constant $c>0$, with
$|\gR|\geq2k$ additionally required for pass@$k$.
Their $\Theta$ rates are worst-case rates over reward classes,
with universal constants. The robust entries are exactly zero
for every finite non-empty reward class, every $k\geq1$,
$\varepsilon>0$, and $\delta\in(0,1)$ under class-wide optimality.
}
\label{tab:sample_complexity_optimal}
\small
\setlength{\tabcolsep}{5pt}
\renewcommand{\arraystretch}{1.6}
\begin{tabular}{@{}lcc@{}}
\toprule
& \textbf{Standard evaluation}
& \textbf{Robust evaluation} \\
\midrule
pass@$1$
&
$\displaystyle
\Theta\!\left(
    \frac{\log(|\gR|/\delta)}{\varepsilon}
\right)$
& $0$
\\[3mm]
pass@$k$, $k\geq2$
&
$\displaystyle
\Theta\!\left[
    \frac1\varepsilon
    \left(
        \frac{\log|\gR|}{\log k}
        +\log\frac1\delta
    \right)
\right]$
& $0$
\\
\bottomrule
\end{tabular}
\end{table}
\end{document}